%% file: main.tex
\documentclass[twoside]{article}

\usepackage[preprint]{aistats2026}
\usepackage[round]{natbib}

\usepackage[utf8]{inputenc} 
\usepackage[T1]{fontenc}    
\usepackage{url}            
\usepackage{booktabs}       
\usepackage{amsfonts}       
\usepackage{nicefrac}       
\usepackage{microtype}      
\usepackage{xcolor}         
\input{Sections/packages}

\begin{document}

%

%

\twocolumn[

\aistatstitle{PAC Learnability in the Presence of Performativity}

\aistatsauthor{ Ivan Kirev \And Lyuben Baltadzhiev \And  Nikola Konstantinov}

\aistatsaddress{ INSAIT,\\ Sofia University\\ ``St. Kliment Ohridski'' \And INSAIT,\\ Sofia University\\ ``St. Kliment Ohridski'' \And INSAIT,\\ Sofia University \\``St. Kliment Ohridski''} ]

\begin{abstract}
  Following the wide-spread adoption of machine learning models in real-world applications, the phenomenon of performativity, i.e. model-dependent shifts in the test distribution, becomes increasingly prevalent. Unfortunately, since models are usually trained solely based on samples from the original (unshifted) distribution, this performative shift may lead to decreased test-time performance. In this paper, we study the question of whether and when performative binary classification problems are learnable, via the lens of the classic PAC (Probably Approximately Correct) learning framework. We motivate several performative scenarios, accounting in particular for linear shifts in the label distribution, as well as for more general changes in both the labels and the features. We construct a performative empirical risk function, which depends only on data from the original distribution and on the type performative effect, and is yet an unbiased estimate of the true risk of a classifier on the shifted distribution. Minimizing this notion of performative risk allows us to show that any PAC-learnable hypothesis space in the standard binary classification setting remains PAC-learnable for the considered performative scenarios. We also conduct an extensive experimental evaluation of our performative risk minimization method and showcase benefits on synthetic and real data.
\end{abstract}

\input{Sections/ch1_introduction}
\input{Sections/ch2_related_work.tex}
\input{Sections/ch3_framework.tex}
\input{Sections/ch4_linear_shift.tex}
\input{Sections/ch5_general_shift.tex}

\input{Sections/ch6_experiments.tex}

\section{Conclusion}
\label{sec:future_work}

We formalized PAC learnability for performative binary classification problems. We studied performative PAC learnability 
for a family of linear performative posterior drifts. We derived a notion of performative empirical risk minimization (PERM) and showed its effectiveness via a theoretical analysis and experiments. 

Given the growing prevalence of performativity in real-world systems, we deem the development and analysis of performative equivalents of parametric classes of distribution shifts (e.g. the linear adjustment model of \cite{maity2024linear}) an exciting direction for future work. Also interesting is integrating real-world performative shifts, e.g. such based on randomized experiments in real contexts \citep{mendler2024engine}. Finally, analyzing the adapted version of RERM theoretically may enable guarantees for more general distribution maps.

\section*{Acknowledgements}

This research was partially funded by the Ministry of Education and Science of Bulgaria (support for INSAIT, part of the Bulgarian National Roadmap for Research Infrastructure).

\bibliography{Sections/bibliography.bib}
\bibliographystyle{plainnat}

\clearpage

\clearpage
\appendix
\thispagestyle{empty}

\onecolumn
\input{Sections/appendix.tex}

\end{document}

%% file: Sections/packages.tex
\usepackage{microtype}
\usepackage{graphicx}
\usepackage{subfigure}
\usepackage{booktabs} 

\usepackage{amsmath}
\usepackage{amssymb}
\usepackage{mathtools}
\usepackage{amsthm}
\usepackage{dsfont}
\usepackage{chngcntr}
\usepackage{ragged2e}
\usepackage{hyperref} 
\usepackage[capitalise,nameinlink]{cleveref}

\theoremstyle{plain}
\newtheorem{theorem}{Theorem}[section]
\newtheorem{proposition}{Proposition}[section]
\newtheorem{lemma}{Lemma}[section]
\newtheorem*{lemma*}{Lemma}
\newtheorem*{theorem*}{Theorem}
\newtheorem*{corollary*}{Corollary}
\newtheorem*{proposition*}{Proposition}

\newtheorem{corollary}{Corollary}[section]
\newtheorem{remark}{Remark}[section]

\theoremstyle{definition}
\newtheorem{definition}{Definition}[section]
\newtheorem{assumption}{Assumptions}[section]

\DeclareMathOperator*{\argmin}{arg\,min}
\newcommand\numberthis{\addtocounter{equation}{1}\tag{\theequation}}
\newcommand{\E}{\mathbb{E}}
\newcommand{\p}{\mathbb{P}}
\newcommand{\R}{\text{\normalfont PR}}

\usepackage[textsize=tiny]{todonotes}

%% file: Sections/ch1_introduction.tex
\section{Introduction}
\label{sec:introduction}

Machine learning (ML) models are increasingly deployed in settings where they have a significant impact on the behaviour of individuals and institutions, thereby affecting the distribution of future data in many social and economic settings \citep{pp1,pp3}. For instance, stock price predictions can guide traders' decisions and hiring algorithms may influence how candidates prepare for and present themselves in interviews. In the ML literature, this effect is studied under the framework of performative prediction \citep{perdomo20a} in which a model induces its own test distribution $\mathcal{D}(\theta)$ and performance is thus evaluated via the \emph{performative risk}, $\mathbb{E}_{z \sim \mathcal{D}(\theta)}[\ell(\theta, z)]$.

A key challenge when learning in performative settings is that the act of model deployment changes the test-time distribution in a model-dependent manner, in particular breaking the standard assumption of supervised learning that the training data is an i.i.d. sample from the test distribution. While specific type of input shifts, specifically adversarial \citep{cullina2018pac} and strategic shifts \citep{IASC}, have been studied in the literature, little is known to date about learnability under general performative effects.

\textbf{Contributions}\hspace{1em} We take a first step towards investigating whether and when performative binary classification problems are learnable, within the probably approximately correct (PAC) learning framework. 

First, we formally define a notion of performative PAC learnability of a hypothesis class, relative to a given performative effect of interest (\cref{sec: framework}). In our setting a learner has access only to samples from the initial distribution and has to learn a classifier that performs well on the distribution it induces via the performative effect under consideration. We also formulate and motivate a natural class of performative effects, which describe conditional label shifts (aka posterior drifts \citep{cai2021transfer,zhu2024label}) that depend linearly on the classifier's outputs. 

Next, we explore performative PAC learnability under this class of performative effects (\cref{sec: linear_shift}). We derive a notion of \emph{performative empirical risk (PER)}, which is computable from the training data and the considered performative effect. We prove that PER is an unbiased estimator of the classifier's performative risk. This allows us to show that any hypothesis space that is PAC-learnable in the standard binary classification setting is also performatively PAC-learnable under linear posterior performative drifts. 

We extend our analysis to the case where the linear posterior performative drift is only known approximately (\cref{sec: imperfect_information}). Furthermore, we show how the idea of performative empirical risk minimization (PERM) can be applied when studying more general performative effects accounting for both posterior drifts and covariate shift (\cref{sec: general_shift}). 

Finally, we empirically evaluate our approach on synthetic data, as well as the Kaggle credit score \citep{GiveMeSomeCredit} and Folktables income prediction \citep{ding2021retiring} datasets (\cref{sec: experiments}). Our results demonstrate the effectiveness of PERM on a range of linear posterior performative drift problems.

%% file: Sections/ch2_related_work.tex
\section{Related Work}
\label{sec: related_work}
\textbf{Performative Prediction}\hspace{1em} The field of performative prediction studies learning problems in which the deployed model affects the distribution it is tested on \citep{perdomo20a}. Usually, works in the area focus on gradient-based optimization methods (e.g. \cite{stochastic_performative,perdomo20a,miller,Drusvyatskiy}), which aim to find performatively optimal or stable models over multiple rounds of model deployment (see \cite{Past_and_future} for a recent overview), while others consider the social implications of performativity in ML \citep{tsoy2025impactperformativeriskminimization}. In contrast to prior work, we focus on PAC learnability in performative classification, with the goal of understanding what learning problems and performative effects are learnable in the classic statistical (one-shot) sense.

\textbf{Strategic and Adversarial Classification} Two types of performative effects have been studied from a PAC learning perspective in the literature, specifically strategic and adversarial input shifts. Strategic classification recognizes that people tend to adapt their features to published decision rules, in order to improve their chances of receiving a desired classification \citep{strategic_classification,social_cost,Zhang_Conitzer_2021}. 
In such works, each test input changes in a way that maximizes the chance of a desirable classification outcome, subject to possible cost/feasibility constraints. PAC learnability has been studied for this specific type of performative shift by  \cite{strategic_pac,IASC,lechner2022learning,lechner2023strategic}. Similarly, adversarial examples \citep{szegedy2013intriguing} can be seen as a specific type of performative shift, where every test input is manipulated within an allowed perturbation set, so as to maximally damage accuracy. PAC learnability under adversarial shifts has also been studied \citep{cullina2018pac,lechner2023adversarially}.

While such works focus on input performative shifts, we also investigate scenarios where the conditional distribution of $Y \mid X$ (as well as the input distribution) may undergo performative shifts, and explore sufficient conditions on the performative effect so that the classification problem remains PAC learnable.

\textbf{Learning under distribution shift}\hspace{1em} Learning under test distribution shift has received significant attention in the machine learning literature (see e.g. \cite{zhou2022domain} for a recent survey). On the theoretical side, (PAC) learnability under distribution shifts has been explored in fields such as domain generalization/adaptation and transfer learning \citep{Daume,Finkel,Mansour,learning_on_diff_domains,Crammer}. These works provide test-time performance guarantees despite distribution shift, usually as a function of the distance between the train and test distribution. However, in these works the test distribution is regarded as fixed, so that it does not change as a function of the learning algorithm and the chosen classifier in particular. Thus, assuming access to some amount of data from this target distribution (or some other prior knowledge about the set of possible test distributions), one can reason about the test time loss of a classifier on this distribution. In constrast, in performative settings, the learning algorithm alters the test distribution (via the learned classifier).

%% file: Sections/ch3_framework.tex
\section{Framework}
\label{sec: framework}
Here we introduce our PAC performative prediction setup, focusing on binary classification tasks.

\subsection{Basic Setup}

We consider binary classification, where the feature space is \(\mathcal{X} \subseteq \mathbb{R}^d\) and the label space is \(\mathcal{Y} = \{-1, 1\}\). A classifier is a function \( h: \mathcal{X} \to \mathcal{Y} \) selected from a hypothesis class \( \mathcal{H} \subset \mathcal{Y}^{\mathcal{X}} \).

\textbf{Performative distribution shift}\hspace{1em}Let $\mathcal{P} (\mathcal{X} \times \mathcal{Y})$ denote the family of distributions over $\mathcal{X} \times \mathcal{Y}$ and $S := \{ (X_i, Y_i)\}_{i=1}^n \in (\mathcal{X} \times \mathcal{Y})^n$ be a dataset of $n$ i.i.d. samples from an underlying distribution $\mathcal{D} \in \mathcal{P} (\mathcal{X} \times \mathcal{Y})$. Based on $S$, a classifier $h$ is trained, which in turn triggers a performative effect that results in a test distribution shift. We formalize the performative effect via a performative map \vspace{-0.1cm}
\begin{equation}
    \tilde{\mathcal{D}} : \mathcal{H} \times \mathcal{P}(\mathcal{X} \times \mathcal{Y}) \to \mathcal{P}(\mathcal{X} \times \mathcal{Y}), \label{eq: distribution_map}
\end{equation}
such that the deployment of \(h\) in an environment originally governed by \(\mathcal{D}\) results in a new distribution \(\tilde{\mathcal{D}}(h, \mathcal{D}) \in \mathcal{P}(\mathcal{X} \times \mathcal{Y})\). We will denote by $\p$ the measure corresponding to the initial distribution $\mathcal{D}$ and by $\tilde{\p}_h$ the measure associated with the shifted distribution $\tilde{\mathcal{D}}(h, \mathcal{D})$. As often $\mathcal{D}$ is clear from the context, we will use $\tilde{\mathcal{D}}(h, \mathcal{D})$ and  $\tilde{\mathcal{D}}(h)$ interchangeably.  In our setting, we are interested in determining if a certain given performative map $\tilde{\mathcal{D}}(\cdot, \mathcal{D})$ (or a set of maps, see Section \ref{sec: imperfect_information}) is learnable, with the challenge being that the distribution $\mathcal{D}$ and hence $\tilde{\mathcal{D}}(h, \mathcal{D})$ are unknown.

\textbf{Performative Risk}\hspace{1em}The \emph{performative risk} of a classifier \(h\) is defined as  
\begin{align*}
    \R(h) &:= \p_h[h(X) \ne Y] 
= \E_{(X, Y) \sim \tilde{D}(h)}[ \mathbf{1}_{\{ h(X) \ne Y \}} ] \\
&=:\E_{\tilde{\p}_{h}}[ \mathbf{1}_{\{ h(X) \ne Y \}} ].
\end{align*}

Minimizing the performative risk \(\R(h)\) over \(h \in \mathcal{H}\) is of central interest in the performative ML literature and poses unique challenges since \(\tilde{\mathcal{D}}(h)\) depends on \(h\) in potentially complex ways (\cite{perdomo20a,jagadeesan22a}).

\textbf{PAC Learnability}\hspace{1em}Now we formalize our definition of performative PAC learnability in binary classification. 
\begin{definition}  
\label{defn:perf_pac_learning}
Let \(\mathcal{H}\) be a hypothesis class and let \(\tilde{\mathcal{D}}(\cdot, \mathcal{D})\) be a distribution map as defined above. We say that \(\mathcal{H}\) is \emph{performatively agnostic PAC-learnable relative to} \(\tilde{\mathcal{D}}(\cdot, \mathcal{D})\) if there exists a function \(n_\mathcal{H} : (0, 1)^2 \to \mathbb{N}\) and a learning algorithm such that for every \(\epsilon, \delta \in (0, 1)\), for any distribution \(\mathcal{D} \in \mathcal{P}(\mathcal{X} \times \mathcal{Y})\), when the algorithm is given \(n \geq n_\mathcal{H}(\epsilon, \delta)\) i.i.d. samples from \(\mathcal{D}\), it returns a hypothesis \(h_S \in \mathcal{H}\) satisfying  
\[
\mathbb{P}_S\left[\R(h_S) - \min_{h \in \mathcal{H}} \R(h) \leq \epsilon\right] \geq 1 - \delta.
\]  
\end{definition}
Similarly to the classic definition of PAC learnability, we require the existence of an algorithm that is a PAC learner, in the sense of selecting a hypothesis that is $\epsilon$-close to optimal, with probability $1-\delta$, given a sufficiently large set of samples. The crucial difference is that this property needs to hold on the classifier-dependent shifted distribution, so both the selected hypothesis $h_S$ and the test distribution depend on the learning algorithm.
\paragraph{Definition rationale} Our notion of performative PAC learnability is defined relative to the distribution map $\tilde{\mathcal{D}}(\cdot, \mathcal{D})$, so certain hypothesis classes may be learnable in some performative settings and not in others. This also means that the learner can be chosen depending on the performative map, i.e. the learner ``knows'' the type of performative shift (though not, of course, the data distribution). This is indeed natural in a statistical (one-shot) setting, where the learner only observes training data drawn from the initial distribution \(\mathcal{D}\), while the deployed classifier is evaluated on the shifted distribution \(\tilde{\mathcal{D}}(h, \mathcal{D})\). If this shifted distribution is arbitrary, PAC learning is impossible (see, in particular, our lower bound in Proposition \ref{proposition: no_free_lunch_lower_bound}). Additionally, it is not feasible to assume access to some target distribution samples here (as in domain adaptation works such as \cite{learning_on_diff_domains}), as the test distribution is only determined once a classifier is selected.

\subsection{Linear Posterior Performative Drift}  
\label{sec: model}

We now motivate and introduce a type of performative shift affecting the conditional label distribution \( Y \mid X \), that we study for most of the paper. Further distribution shifts, including such affecting the marginal distribution of \( X \), are discussed in \cref{sec: general_shift}. We start by introducing the following motivating examples. 

\textbf{Placebo Effect}\hspace{1em}In a medical diagnosis scenario, consider a classifier \( h \) that predicts disease progression based on patient features \( X \). Suppose \( h(X) = 1 \) indicates a prediction that the disease will progress, while \( h(X) = -1 \) suggests remission. Upon receiving such a diagnosis, the patient's actual disease outcome may be influenced by a placebo effect—for instance, stress induced by a negative prognosis. This psychological response can alter the probability of remission, potentially even affecting the true course of the disease. This exemplifies a shift in the conditional label distribution \( Y \mid X \) due to the deployment of the classifier.

\textbf{Traffic Navigation Systems}\hspace{1em}In a traffic modelling scenario, let a classifier \( h \) predict whether a route will be busy or not in the near future, based on features such as the current road conditions. When a particular route with features $X$ is suggested as not busy by a classifier, denoted by $h(X) = 1$, and this classifier is deployed (for example via a navigation application), more drivers may choose the route, resulting in changes in the road conditions and reducing the optimality of the selected route. 

Inspired by these examples, we introduce a type of performative posterior shift, in which the conditional label distribution changes linearly as a function of the prediction for each instance. Specifically, the conditional probability of \( \{Y = 1 \mid X = x \}\) shifts as a function of the classifier's output \(h(x)\) as:
\begin{align*}  
    \tilde{\p}_h[Y = 1 \mid X = x] = &\alpha(h(x)) \p[Y = 1 | X = x]\\
    & + \beta(h(x)), \numberthis \label{eq: linear_model}
\end{align*}
where $\alpha(h(x)) = 
\begin{cases}  
    a_1, & \text{if } h(x) = 1, \\  
    a_3, & \text{if } h(x) = -1,  
\end{cases}$ and 

$\beta(h(x)) =
\begin{cases}  
    a_2, & \text{if } h(x) = 1, \\  
    a_4, & \text{if } h(x) = -1.
\end{cases}$ 

The constants \(a_1, a_2, a_3, a_4\) control the strength and direction of the performative effect. 
\begin{remark}
    \label{remark1}
    The constants \(a_1, a_2, a_3, a_4\) must satisfy the following constraints to ensure that the shifted probabilities remain well-defined: $a_2, a_4 \in [0, 1],$ $a_1 \in [-a_2, 1 - a_2],$ $a_3 \in [-a_4, 1 - a_4]$. By varying these parameters, both positive and negative performative effects can be modelled, meaning that the classifier's output may increase or decrease the likelihood of a correct label depending on the setting.
\end{remark}

To illustrate the linear performative posterior drift, we return to our examples. 

\textbf{Example instantiations}\hspace{1em}In the absence of a performative effect, i.e., when for all \( h \in \mathcal{H} \), \( \tilde{\mathcal{D}}(h) = \mathcal{D} \), we take \( a_1 = a_3 = 1 \) and \( a_2 = a_4 = 0 \), which recovers the standard binary classification framework.

In the example of the placebo effect, suppose that after receiving a diagnosis, the probability that the disease progresses ($Y=1$) for a certain patient with features $x$ increases linearly if their diagnosis is pessimistic ($h(x) = 1$) and stays the same if their diagnosis is optimistic ($h(x) = -1$). This can be modelled in the framework of \cref{sec: model} by taking $a_1 = 1-a, a_2 = a, a_3 = 1, a_4 = 0,$ where $a \in [0, 1]$ reflects the strength of the placebo effect. 

In the context of traffic navigation systems suppose we classify the presence of heavy traffic (\( Y = 1 \)) on a given street with features \( X \) using a classifier \( h \in \mathcal{H} \), which is then deployed in a navigation application. After receiving a recommendation from the navigating system, less people will choose roads classified as $h(x) = 1$, which will reduce the actual traffic there. Similarly, if a road is classified as not busy, the traffic there is likely to increase based on this classification. This can be modelled by taking $a_1 = a, a_2 = 0, a_3 = 1 - b, a_4 = b$, where $a \in [0, 1]$ represents the reduction factor in traffic when the classifier predicts heavy traffic, and $b \in [0, 1]$ reflects the increase in probability of heavy traffic if the classifier predicts no traffic. 

%% file: Sections/ch4_linear_shift.tex
\section{PAC Learnability under the Linear Performative Posterior Drift}
\label{sec: linear_shift}

In this section we analyse the linear performative posterior drift described in \cref{sec: model}. First we derive an unbiased estimate of the risk of a hypothesis under that performative effect, which we term the \textit{performative empirical risk (PER)}. We then consider minimizing PER as a learning algorithm and show that this is a PAC learner whenever the hypothesis space is PAC-learnable in the standard classification setting. Finally, we explore an imperfect information case, where the performative map $\tilde{\mathcal{D}}$ is only partially known.

\subsection{Performative Empirical Risk}

In our setting, we seek to minimize the performative risk \(\R(h)\) over \(h \in \mathcal{H}\), using only samples from the initial distribution \(\mathcal{D}\). We first show that under the linear performative posterior drift, it is possible to construct a \emph{performative empirical risk (PER)} estimate that serves as an unbiased estimator of \(\R(h)\), using only the samples from the original distribution \(\mathcal{D}\). 
\begin{lemma}
    \label{lemma: performative_risk}
    Let \(S = \{ (X_i, Y_i) \}_{i=1}^n\) be a set of \(n\) i.i.d.\ samples from \(\mathcal{D}\), and let \(h \in \mathcal{H}\) be a hypothesis. Suppose \(\tilde{\mathcal{D}}(h)\) is the distribution obtained from \(\mathcal{D}\) after a performative linear shift as defined in \cref{eq: distribution_map}. Then the performative risk \(\R(h)\) is given by
    \begin{align*}
        \R(h)= &(1-a_2) \p[h(X) = 1]
        - a_1 \p[Y=1, h(X)=1]\\
        &+ a_3 \p[Y=1, h(X)=-1] \numberthis \label{eq: risk}\\
        &+ a_4 \p[h(X)=-1].  
    \end{align*}
    Moreover, an unbiased estimator of \(\R(h)\) is the performative empirical risk \vspace{-0.15cm}
    \begin{equation}
        {\R}_n(h) := \frac{1}{n} \sum_{i=1}^n \big(\alpha_1 h(X_i) + \alpha_2 Y_i + \alpha_3 Y_i h(X_i) + \alpha_4 \big),  \label{eq: empirical_risk}
    \end{equation}
    where \(\alpha_1 = \frac{2 - a_1 - 2a_2 -a_3 - 2a_4}{4}, \ \alpha_2 = \frac{a_3 - a_1}{4}, \
        \alpha_3 = \frac{-a_1 - a_3}{4}, \ \alpha_4 = \frac{2 - a_1 - 2a_2 + a_3 + 2a_4}{4}.\)
\end{lemma}

The expression ${\R}_n(h)$ is an average of empirical terms, one for each data point, and estimates the performative risk. We therefore term ${\R}_n(h)$ the performative empirical risk (PER). We denote the real-valued functions indexed by \(h\) that correspond to the individual terms in PER as:
\begin{equation*}
    f_h(X_i, Y_i) := \alpha_1 h(X_i) + \alpha_2 Y_i + \alpha_3 Y_i h(X_i) + \alpha_4,
\end{equation*}
and define the corresponding function class
\begin{align*}
    \mathcal{F}_{\mathcal{H}} &:= \left\{f_h(X, Y) : h \in \mathcal{H} \right\}\\
    &= \left\{ \alpha_1 h(X) + \alpha_2 Y + \alpha_3 Y h(X) + \alpha_4 \, : \, h \in \mathcal{H} \right\}.
\end{align*}

\begin{remark}
    The analysis in this and subsequent sections applies to any setting where the empirical risk admits the form $\R_n(h) = \frac{1}{n} \sum_{i=1}^n f_h(X_i, Y_i)$, for some \(f_h \in \mathcal{F}_{\mathcal{H}}\). In particular, the results do not rely on the specific values of the coefficients \(\{\alpha_i\}_{i=1}^4\) as in \cref{remark1}, but rather hold for any fixed real constants defining the function class \(\mathcal{F}_{\mathcal{H}}\).
\end{remark}

\begin{remark}
    Since \(Y, h(X) \in \{-1, 1\}\), any higher-order term involving these variables reduces to either a constant or one of the basic terms \(h(X)\), \(Y\), or \(Y h(X)\). As a result, any linear combination of such higher-order terms is already contained within the function class \(\mathcal{F}_{\mathcal{H}}\).
\end{remark}

\subsection{Generalization Bounds}
To ensure that minimizing PER leads to low true performative risk, we now derive generalization bounds for the performative risk, for the hypothesis class \(\mathcal{H}\). As in standard binary classification, our bounds will depend on a measure of the complexity of the hypothesis space. In our work, we will use the (data-dependent) notion of complexity known as Rademacher complexity \citep{shalev2014understanding}.

\begin{definition}[Rademacher Averages]
    \label{def: rademacher_averages}
    For a class $\mathcal{F}$ of functions and a data sample $S = \{(x_i, y_i)\}_{i=1}^n$, the Rademacher average and the conditional Rademacher average are defined respectively as \vspace{-0.2cm}
    \begin{align*}
        \mathfrak{R}(\mathcal{F}) = \E \bigg[\sup_{f \in \mathcal{F}} \frac{1}{n} \sum_{i=1}^n \sigma_i f(X_i, Y_i) \bigg],\\
        \mathfrak{R}_n(\mathcal{F}) = \E_\sigma \bigg[ \sup_{f \in \mathcal{F}} \frac{1}{n} \sum_{i=1}^n \sigma_i f(x_i, y_i) \bigg],
    \end{align*}
    where $\E_\sigma$ denotes the expectation taken over the Rademacher random variables $\{\sigma_i\}_{i=1}^n$ (i.i.d. random variables with $\p[\sigma_i = 1] = \p[\sigma_i = -1] = \frac12$), and $\E$ denotes expectation with respect to $\sigma$ and the data. 
\end{definition}
We now state our following generalization result.
\begin{theorem}
    \label{theorem: infinite_hypothesis}
    Let $\mathcal{H}$ be a hypothesis class, $S = \{(x_i, y_i)\}_{i=1}^n$ be $n$ i.i.d. samples from $\mathcal{D}$, and let $\R_n(h)$ be defined as in \cref{eq: empirical_risk}. Then, given a distribution map $\tilde{D}(\cdot)$ defined as in \cref{eq: linear_model}, for any $\delta > 0$ with probability at least $1-\delta$ we have that for all $h \in \mathcal{H}$ \vspace{-0.15cm}
    \begin{align*}
       |\R&(h) - \R_n(h)| \le 2 (|\alpha_1| + |\alpha_3|) \mathfrak{R}_n(\mathcal{H}) \\
        &+ 2((1 + \sqrt{2})|\alpha_1| + |\alpha_2| + (1 + \sqrt{2})|\alpha_3|)\sqrt{\frac{\ln{\frac{4}{\delta}}}{2n}}.
    \end{align*}  
\end{theorem}
\vspace{-0.2cm}The following statement relating the optimal hypothesis with the one produced by minimizing the PER follows from the theorem above.
\begin{corollary}
    \label{corollary: pac_learnability_infinite}
    In the setting of \cref{theorem: infinite_hypothesis}, let $h_S^\ast \in \argmin_{h \in \mathcal{H}} \R_n(h)$. Then with probability $\geq 1-\delta$ \vspace{-0.25cm}
    \begin{align*}
        \R(&h_S^\ast) \le \min_{h \in \mathcal{H}} \R(h) + 4 (|\alpha_1| + |\alpha_3|) \mathfrak{R}_n(\mathcal{H})\\
        &+ 4((1 + \sqrt{2})|\alpha_1| + |\alpha_2| + (1 + \sqrt{2})|\alpha_3|)\sqrt{\frac{\ln{\frac{4}{\delta}}}{2n}}.
    \end{align*}
    In particular, if $\mathcal{H}$ is agnostic PAC-learnable, then it is performatively agnostic PAC-learnable w.r.t. the distribution map $\tilde{\mathcal{D}}(\cdot)$. 
\end{corollary}
The last statement follows from the bound, since $\mathfrak{R}(\mathcal{H}) \leq \mathcal{O}\left(\sqrt{\frac{VC(\mathcal{H})}{n}}\right)$ \citep{shalev2014understanding} and $VC(\mathcal{H})$ is finite whenever $\mathcal{H}$ is PAC-learnable.

We can also use the bound to reason about the hardness of the performative problem, compared to standard binary classification. In particular, the Rademacher complexity term \( \mathfrak{R}_n(\mathcal{H}) \) captures the capacity of the hypothesis class to fit noise, and its coefficient \( 4(|\alpha_1| + |\alpha_3|) \) reflects how strongly the complexity of $\mathcal{H}$ affects generalization. In the traffic modelling example from \cref{sec: framework}, this coefficient is smaller than in the non-performative case (where the Rademacher term features with a constant weight of $2$ \citep{shalev2014understanding}), suggesting that the complexity of \(\mathcal{H}\) plays a lesser role compared to the non-performative version of the problem.

\subsection{Imperfect Information about the Distribution Map}
\label{sec: imperfect_information}

So far we have studied PAC learnability with respect to a fixed performative map $\tilde{\mathcal{D}}(\cdot)$. In this section we explore learnability relative to a set of possible performative maps, specifically situations where the parameters of the map (\ref{eq: linear_model}) are only known to be inside a certain interval. Let $a_i \in I_i, \ i \in \{1, \dots, 4 \},$ where $I_i$ are known intervals with length $\epsilon_i \ge 0$, and let $\epsilon := \sum_{i=1}^4 \epsilon_i$. Given the intervals $I_i$, the performative empirical risk is \vspace{-0.15cm}
\begin{equation*}
    \overline{\R}_n(h) = \frac{1}{n} \sum_{i=1}^n \big( \overline{\alpha}_1 h(X_i) + \overline{\alpha}_2 Y_i + \overline{\alpha}_3 Y_i h(X_i) + \overline{\alpha}_4 \big),
\end{equation*}
where $\overline{\alpha}_1 = \frac{2 - \overline{a}_1 - 2\overline{a}_2 -\overline{a}_3 - 2\overline{a}_4}{4}, \overline{\alpha}_2 = \frac{\overline{a}_3 - \overline{a}_1}{4},
    \overline{\alpha}_3 = \frac{-\overline{a}_1 - \overline{a}_3}{4}, \overline{\alpha}_4 = \frac{2 - \overline{a}_1 - 2\overline{a}_2 + \overline{a}_3 + 2\overline{a}_4}{4},$ and $\overline{a}_i$ are any fixed points within the intervals $I_i$. The distribution map $\tilde{\mathcal{D}}(\cdot)$ is defined as in \cref{sec: model}, in which the parameters $\{a_i\}_{i=1}^4$ are any constants within the intervals $I_i$, unknown to the learner. We have the following result:
\begin{theorem}
    \label{theorem: imperfect_information}
    Let $\mathcal{H}$ be a hypothesis class, $S = \{(x_i, y_i)\}_{i=1}^n$ be $n$ i.i.d. samples from $\mathcal{D}$, and let $\R(h)$ and $\R_n(h)$ be defined as in \cref{eq: risk} and \cref{eq: empirical_risk}. For the map $\tilde{\mathcal{D}}(\cdot)$ defined as above, if $h_S^\ast \in \argmin_{h \in \mathcal{H}} \overline{\R}_n(h)$, then with probability at least $1-\delta$
    \begin{align*}
        \R(h_S^\ast) \le \min_{h \in \mathcal{H}} \R(h) + 2\epsilon + 2 A \mathfrak{R}_n(\mathcal{H})
         + 2B \sqrt{\frac{\ln{\frac{4}{\delta}}}{2n}},
    \end{align*}
    where $A = 2 (|\alpha_1| + |\alpha_3|)$ and $B = 2((1 + \sqrt{2})|\alpha_1| + |\alpha_2| + (1 + \sqrt{2})|\alpha_3|)$.
\end{theorem}

The bound features a term of $\mathcal{O}\left(\epsilon\right)$, reflecting the uncertainty in the specification of the distribution map $\tilde{\mathcal{D}}(\cdot)$. If the proxies $\overline{a}_i$ are chosen as midpoints of $I_i$, the error term $2 \epsilon$ can be reduced by half to $\epsilon$ in the upper bound. Thus, in particular, PER can only recover an $\mathcal{O}\left(\epsilon\right)$-optimal hypothesis even as $n\to\infty$.

Intuitively, this term is irreducible, as the learner only has access to data from the training distribution and partial knowedge about the performative map. The proposition below formalizes this intuition, by showing that the linear dependence on $\epsilon$ is inherent, highlighting how parameter imprecision affects generalization in this performative setting.

\begin{proposition}
    \label{proposition: no_free_lunch_lower_bound}
Consider an arbitrary input-output space $\mathcal{X} \times \mathcal{Y}$ and a hypothesis space $\mathcal{H}$ containing at least two distinct functions. For any learning algorithm $\mathcal{A}$, there exists a distribution $\mathcal{D}$ and a set of true parameters ${a}_i$, that lie within known intervals ${I}_i$, of total length $\epsilon = \sum_{i=1}^{d} \epsilon_i$, such that with probability at least $1/2$ over the random choice of the training set $S \sim \mathcal{D}^{ n}$ and the possible randomness of the algorithm, the excess risk of the algorithm's output ${h}_S$ is large: 
\begin{equation*}
    \R({h}_S) - \min_{h \in \mathcal{H}} \R(h) \geq \frac{\epsilon}{4}.
\end{equation*}
\end{proposition}

%% file: Sections/ch5_general_shift.tex
\section{Learnability under Other Performative Maps}
\label{sec: general_shift}

We now consider cases where the performative effect can impact both the feature distribution $X$ and the conditional distribution $Y \mid X$. We show conditions under which the performative task remains learnable.

Specifically, if \(\tilde{\p}_h\) is \emph{absolutely continuous} with respect to \(\p\), we can write the performative risk as \vspace{-0.05cm}
\begin{equation*}
    \tilde{\p}_h[h(X) \ne Y] = \E_{\p} \bigg[ \frac{d \tilde{\p}_h}{d \p}(X, Y) \mathbf{1}_{\{ h(X) \ne Y \}} \bigg], 
\end{equation*}
where the expectation is taken with respect to the measure \(\p\), and \(\frac{d \tilde{\p}_h}{d \p}(X, Y)\) is the \emph{Radon-Nikodym (RN) derivative}, describing the relative density of the shifted distribution \(\tilde{\p}_h\) with respect to the original distribution \(\p\) \citep{billingsley2012probability}. We note that such reweighting techniques involving the RN derivative (usually in the form of density ratios) are common in domain adaptation \citep{huang2006correcting,sugiyama2007covariate}. There, however, the test distribution is fixed and independent of the selected classifier and the ratio can in particular be estimated given some data from the (fixed) target distribution. We require:
\begin{assumption}
  \label{assumption: general_shift}
  We assume the following conditions hold for all \(h \in \mathcal{H}\): (i) the measure \(\tilde{\p}_h\) is absolutely continuous with respect to \(\p\); (ii) the RN derivative is uniformly bounded—specifically, there exists a constant \(M \ge 0\) such that for all \((x, y) \in \mathcal{X} \times \mathcal{Y}\), \(\frac{d \tilde{\p}_h}{d \p}(x, y) \leq M\); and (iii) the function $\frac{d\tilde{\mathbb{P}}_h}{d\mathbb{P}}$ is uniquely determined via the values of $h$, i.e. that $\frac{d\tilde{\mathbb{P}}_h}{d\mathbb{P}}(x,y) = \psi(x,y,h(x))$ for some function $\psi$.\footnote{The third assumption was added in the second version of this manuscript, to fix an incorrect bounding argument in Step 3 of the initial proof of Theorem \ref{theorem: general_shift}. We apologise for this unintentional mistake and thank an anonymous reviewer for detecting the issue in the previous version.}
\end{assumption}
The first assumption is a standard condition ensuring the existence of the RN derivative. The second assumption bounds the strength of the performative effect. The third one states that the change in distribution only depends on the classifier via its values and not, say, on its parameters. Note that, as in Section \ref{sec: linear_shift}, the learners we study below depend on the distribution map (via the RN derivative), as means towards proving the existence of a PAC learner in the sense of Definition \ref{defn:perf_pac_learning}.

Given a hypothesis class $\mathcal{H}$, we define the class \(\Omega_{\mathcal{H}}\) as the set of maps \(\omega_h: \mathcal{H} \times(\mathcal{X} \times \mathcal{Y}) \to \mathbb{R}_{+}\) of the form $\Omega_{\mathcal{H}} = \{ \omega_h(x, y) = \frac{d \tilde{\p}_h}{d \p}(x, y)  \mathbf{1}_{\{ h(X) \ne Y \}}$, $:  h \in \mathcal{H}, (x, y) \in \mathcal{X} \times \mathcal{Y} \}$. Now estimating the performative risk $\R(h)$ can be done by using the following unbiased estimator of the performative risk $\R(h)$ 
\begin{align*}
    \R_n(h) &= \frac{1}{n} \sum_{i=1}^n \omega_h(X_i, Y_i) \\
    &= \frac{1}{n} \sum_{i=1}^n  \frac{d \tilde{\p}_h}{d \p}(X_i, Y_i)  \mathbf{1}_{\{ h(X) \ne Y \}}, \numberthis \label{eq: rn_empirical_risk}
\end{align*}
where \cref{assumption: general_shift} ensure this can be computed using only data from the initial distribution $\mathcal{D}$.
\begin{theorem}
  \label{theorem: general_shift}
  Let $\mathcal{H}$ be a hypothesis class, $S = \{(x_i, y_i)\}_{i=1}^n$ be $n$ i.i.d. samples from $\mathcal{D}$, and let $\R(h)$ and $\R_n(h)$ be defined as \cref{eq: risk} and \cref{eq: rn_empirical_risk}. If \cref{assumption: general_shift} are satisfied, taking $h_S^\ast \in \argmin_{h \in \mathcal{H}} \R_n(h)$, for any $\delta > 0$ with probability at least $1-\delta$ we have that 
    \begin{align*}
        \R(h_S^\ast) \le \min_{h \in \mathcal{H}} \R(h) &+ 2M \mathfrak{R}_n(\mathcal{H})\\
      &+ (4 + 2\sqrt{2})M\sqrt{\frac{\ln{\frac{4}{\delta}}}{2n}}.
    \end{align*}
    In particular, if $\mathcal{H}$ is agnostic PAC-learnable, then it is performatively agnostic PAC-learnable w.r.t. $\tilde{\mathcal{D}}(\cdot)$. 
\end{theorem}

\textbf{Strategic Shift Example} \hspace{1em}During a hiring interview, candidates are assessed based on their features \( X \), e.g. their CV or test performance. Once the interview process of a firm becomes known, candidates may prepare by studying specific topics or memorizing common questions to increase their chances of being classified as hired (\( h(X) = 1 \)). This can lead to shifts in both the feature distribution \( X \) (as candidates show improvement in the assessed skills) and the conditional label distribution \( Y \mid X \) (as memorizing common questions may enhance interview performance in a superficial way, without necessarily improving the required underlying abilities). Specifically, let the features $X$ of a given candidate be initially distributed as $X \sim \mathcal{N}(\mu_1, \sigma_1)$, and after the performative effect takes place, their distribution shifts to $X \sim \mathcal{N}(\mu_2, \sigma_2)$ if $h(X)=-1$ and remains unchanged otherwise. Then,\vspace{-0.05cm}
\begin{equation*}
  \frac{d \tilde{\p}_h}{d \p}(x) = 
  \begin{cases}
    1, & h(x) = 1,\\
    \frac{\sigma_1}{\sigma_2} \exp \big( \frac{(x - \mu_1)^2}{2\sigma_1^2} - \frac{(x - \mu_2)^2}{2\sigma_2^2}  \big) , & h(x) = -1.
  \end{cases}
\end{equation*}
This is uniformly bounded over $\mathcal{X}$ if $\sigma_2 < \sigma_1$. Additionally, the conditional distribution $\p(Y \mid X = x)$ may decrease, which can be modelled by our linear shift model from \cref{sec: model} by setting $a_1 = a \in [0, 1]$, $a_2 = a_3 = a_4 = 0$, reflecting an increased false positive rate under the new distribution. Thus, the RN derivative for hypothesis $h$ is given by: \vspace{-0.05cm}
\begin{align*}
  \frac{d \tilde{\p}_h}{d \p}(x, y) &= \frac{d \tilde{\p}_h}{d \p}(x) \frac{d \tilde{\p}_h}{d \p}(y \mid X = x) \\
  &= 
  \begin{cases}
    a, & h(x) = 1,\\
    \frac{\sigma_1}{\sigma_2} \exp \big(\frac{(x - \mu_1)^2}{2\sigma_1^2}  -  \frac{(x - \mu_2)^2}{2\sigma_2^2} \big) , & h(x) = -1.
  \end{cases}
\end{align*}

%% file: Sections/ch6_experiments.tex
\section{Experiments}
\label{sec: experiments}
We evaluate the performance of PERM on synthetic and semi-synthetic data, under a range of linear performative posterior drifts. We compare it to standard empirical risk minimization (ERM), as well as a method inspired by the repeated risk minimization algorithm \citep{perdomo20a}.

\subsection{Minimizing PER}
\label{sec: surrogate}
We begin by describing how PERM is implemented in practice using a surrogate loss.
Directly minimizing the performative risk is challenging due to the non-smooth nature of binary classification. We seek a function \( f : \mathcal{X} \to \{-1, 1\} \) predicting a label \( y \) for input \( x \). In standard classification, since the risk depends only on \( y f(x) \), the zero-one loss is typically approximated by the following surrogate loss, namely the \emph{logistic loss}:
\begin{equation*}
    V(f(x), y) = \phi(y f(x)) := \log(1 + e^{-y f(x)}) / \log(2), 
\end{equation*}
This loss upper bounds the zero-one loss and is convex and differentiable \citep{logistic_loss}, thus enabling gradient-based optimization.

In the performative setting, however, the loss depends separately on the classifier and the label, not just their product. To accommodate this, we propose the following surrogate for the performative risk from \cref{lemma: performative_risk}: \vspace{-0.2cm}
\begin{align*}
    V(f(x),y) = &\alpha_1 (1 - 2 \phi(f(x))) + \alpha_2 y \\
     &+ \alpha_3 (1 - 2 \phi(y f(x))) + \alpha_4, \numberthis \label{eq: surrogate_loss}
\end{align*}
where \( \phi(y f(x)) \) accounts for \( \frac{1 - y h(x)}{2} \), and \( \phi(f(x)) \) (with \( y = 1 \)) for \( \frac{1 - h(x)}{2} \). This surrogate loss captures all classifier-dependent terms and serves as a convex and differentiable upper bound for the performative risk in \cref{lemma: performative_risk}, whenever \( \alpha_1 \leq 0 \) and \( \alpha_3 \leq 0 \). Note that in the non-performative case, setting $\alpha_1 = \alpha_2 = 0$, $\alpha_3 = -1/2$, and $\alpha_4 = 1/2$ recovers the usual log loss.

\subsection{Experiments}
\label{subsec: experiments}

We evaluate the PERM method on a synthetic dataset, as well as the 
Credit scoring (\cite{GiveMeSomeCredit}), and the U.S. Folktables (\cite{ding2021retiring}) datasets. Each experiment simulates a performative setting via the linear conditional label shift model from \cref{sec: model}, with different parameters $a_1, a_2, a_3, a_4$. We report average accuracy on the performative test set over 10 runs, along with one standard deviation confidence intervals, as shown in \cref{fig: MAIN_TEXT_PLOT}. Further details and additional experiments using alternative parameters $a_1,\ldots, a_4$ for the distribution map are presented in the supplementary material, \cref{appendix: A}.

In all settings, ERM optimizes the standard risk $\p[Y \neq h(X)]$, while PERM minimizes the full performative risk $\tilde{\p}_h[Y \neq h(X)]$, which accounts for distributional shifts as per \cref{lemma: performative_risk}.

\textbf{Synthetic Data}\hspace{1em}We generate 5{,}000 samples with features $X = (x_1, x_2)$ sampled uniformly from $[-3, 3]^2$, and labels assigned via $P[Y = 1 \mid X] = \sigma(0.5 x_2 - 0.2 x_1^2)$, where $\sigma$ is the sigmoid function. A linear model is trained on 80\% of the data. The performative shift follows the placebo scenario in \cref{sec: model}, with $a_1 = 1-a$, $a_2 = a$, $a_3 = 1$, $a_4 = 0$.

\begin{figure*}
    \includegraphics[width=\textwidth]{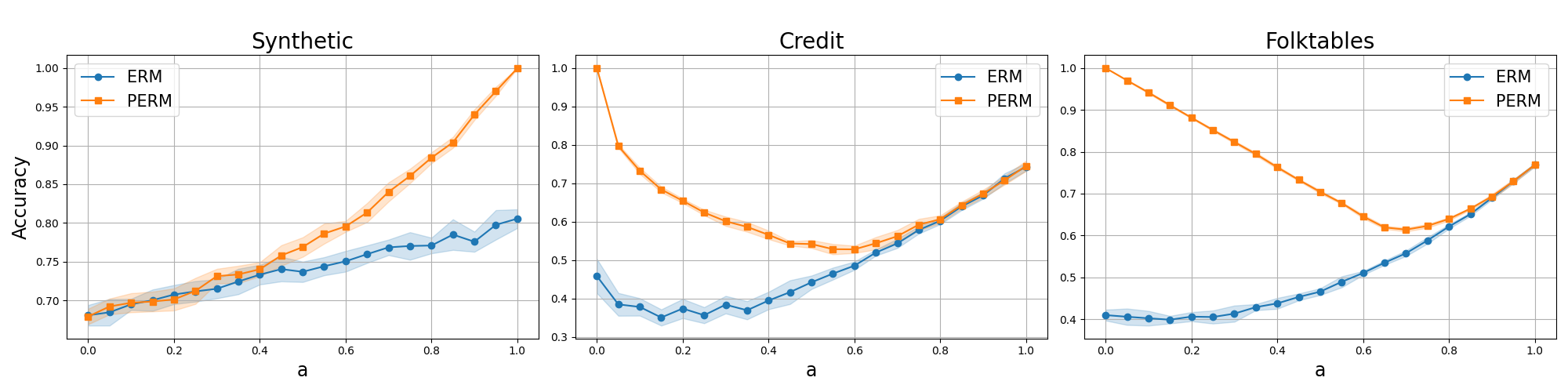}
    \caption{Accuracy on the performative test set as a function of effect strength $a$.}
    \label{fig: MAIN_TEXT_PLOT}
\end{figure*}

\begin{figure*}
    \includegraphics[width=\textwidth]{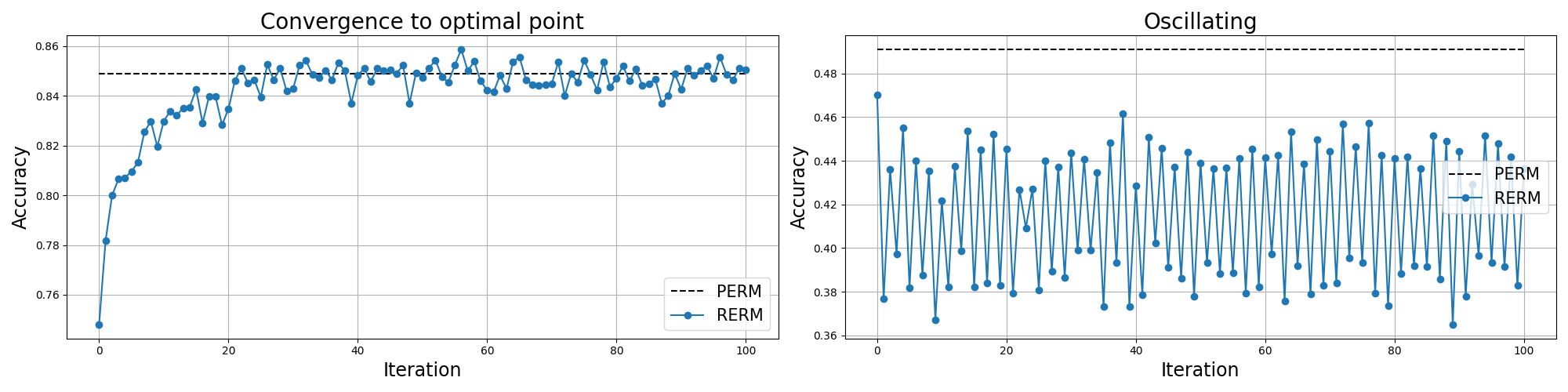}
    \caption{Behavior of the RERM-inspired method on the credit score dataset.}
    \label{fig:rerm_results}
\end{figure*}

\textbf{Credit Score Data}\hspace{1em}We use a class-balanced subset of the credit score dataset, with $X$ including features like monthly income and number of credit lines, and $Y \in \{1, -1\}$ indicating loan default. A one-layer neural network is trained. The performative shift uses $a_1 = a$, $a_2 = 1-\sqrt{a}$, $a_3 = a^3$, $a_4 = 1 - a^2$. To construct the performative test set, we estimate  the true underlying probability $P(Y \mid X)$ using an \textit{independent} random forest classifier with 18 trees and depth 8, trained on the full dataset. Each training is repeated 10 times with regenerations of the sampling of the balanced dataset and the simulation of the performative effect. 

\textbf{Folktables Data}\hspace{1em}We use a class-balanced subset of the Folktables dataset, specifically data for the state of California for 2018, where $X$ includes features such as class of worker, educational attainment, and hours worked per week, and $Y \in \{1, -1\}$ indicates whether income exceeds \$50{,}000. A two-layer neural network is trained, with a $30 \%$ validation and $30 \%$ test splits for hyperparameter tuning. The performative shift uses $a_1 = a$, $a_2 = 1-a$, $a_3 = a^2$, $a_4 = 1-a^2$, and label probabilities are estimated using the same random forest model as above on the full dataset. Each training is repeated 10 times with a regeneration of the balanced dataset and the simulation of the performative effect.

\textbf{Results}\hspace{1em}We observe that the PERM method (naturally) outperforms ERM, while also giving consistent results across repeats. For small performative effects ($a \approx 0$ for synthetic, $a \approx 1$ for real data), ERM and PERM perform similarly. As the strength of the performativity increases, PERM outperforms ERM by explicitly optimizing for the performative risk. When performativity is strong ($a \approx 1$ for synthetic, $a \approx 0$ for real data), PERM obtains accuracy up to $1$ by learning to predict a label which is assigned to each instance with probability $1$ due to performativity.

\subsection{Repeated Empirical Risk Minimization}
\label{section: RERM}

Here we investigate whether a popular performative learning method from prior work, namely the Repeated Empirical Risk Minimization (RERM) method of \cite{perdomo20a}, can be adapted within our framework. RERM addresses performative effects by iteratively retraining a model on data sampled from the distribution induced by the previously deployed model. Unlike the original RERM setup, which requires fresh samples from the shifted distribution at each iteration, our framework assumes access to only a single training dataset from the original distribution, along with knowledge of the performative distribution map. To adapt RERM to this setting, we fix the training features and update the labels at each iteration according to the distribution induced by the classifier from the previous step. This simulates new performative distributions without additional sampling.

We apply RERM to the credit score dataset from \cref{subsec: experiments}, using a performative label shift based on a learned classifier and parameterized via the linear shift in \cref{sec: model}. The model is trained using logistic loss, and the process repeats multiple times.
Results are shown in \cref{fig:rerm_results}. In the first scenario (Left), using parameters $a_1 = 0.1$, $a_2 = 0.8$, $a_3 = 0.2$, $a_4 = 0.3$, RERM converges to the same classifier as PERM, demonstrating that under moderate shifts and stable mappings, RERM can recover the performative optimum. In the second scenario (Right), with $a_1 = 0.8$, $a_2 = 0$, $a_3 = 0.48$, $a_4 = 0.52$, RERM fails to converge and oscillates between two classifiers. This illustrates that even with full knowledge of the performative map, convergence is not guaranteed. Note that the accuracy of PERM is acquired without the need for iterative training, but is nonetheless plotted for comparison.

%% file: Sections/appendix.tex
\begin{center}
    {\LARGE Supplementary Material}
\end{center}

\appendix
\begin{itemize}
    \item \cref{appendix: A} contains further details on the experiments from \cref{sec: experiments} and additional experiments containing different variations of the linear distribution shift model in \cref{sec: model}.
    \item \cref{appendix: B} contains additional results regarding generalization bounds for finite hypothesis classes $\mathcal{H}$ in the setting of \cref{sec: linear_shift}.
    \item \cref{appendix: C} contains the proofs of all results from the main text. 
    \item \cref{appendix: D} contains information about compute resources used for the experiments.
\end{itemize}

\section{Experiments}
\label{appendix: A}

\subsection{Additional Details on Experiments}
\label{appendix: A.1}

\textbf{Synthetic Data}\hspace{1em}We use a simple linear model which receives 2-dimensional inputs and produces a scalar output with no activation function. This corresponds to a logistic regression classifier. The learning rate is initialized at $0.01$ and adjusted using a \texttt{ReduceLROnPlateau} scheduler monitoring the training loss, with a reduction factor of $0.9$ and a patience of $5$ epochs. The model is trained for $25$ epochs on the training set using the default batch size of $32$.

\textbf{Credit Score Data}\hspace{1em}The model is a single-hidden-layer neural network with 16 neurons, ReLU activations, and L2 regularization. The input consists of 10 features. Training uses the Adam optimizer with an initial learning rate of $0.01$ and the same learning rate scheduler as above. Depending on the setting, either a standard logistic loss or a custom performative loss is used. Training is run for $25$ epochs per trial with a batch size of $32$.

\textbf{Folktables Data}\hspace{1em}The model is a two-hidden-layer neural network with 64 and 16 neurons respectively, L2 regularization, and ReLU activations. The model uses the Adam optimizer with a learning rate of $0.05$ and a \texttt{ReduceLROnPlateau} scheduler. It is trained for $25$ epochs with a batch size of $1024$. A custom callback logs validation performance per epoch. When the performative effect is enabled, a surrogate performative loss is used in place of logistic loss. The dataset is split in two stages, using a 70\%-30\% ratio at each step to create training, validation, and test sets.

For each case, the custom loss function, described in \cref{sec: surrogate}, is compiled depending on whether the performative effect is enabled. Otherwise, the standard cross-entropy loss is used.

\subsection{Additional Experiments}
    
In this section, we present additional experiments exploring the impact of different parametrizations of the linear distribution shift model described in \cref{sec: model}. Specifically, we consider five different settings for the parameters $a_1, a_2, a_3$, and $a_4$, which we apply across all three datasets used in our main experiments: Synthetic, Credit Score, and Folktables. In all cases, the rest of the experimental setup, including the model architecture, training procedure, and evaluation metrics, are kept consistent with those described in \cref{sec: experiments} and \cref{appendix: A.1}
    
Each of the five performative models is defined as a function of the scalar parameter $a \in [0, 1]$:
    
\begin{enumerate}
        \item Model 1: $a_1 = a,\quad a_2 = 0.75 (1 - a),\quad a_3 = 1,\quad a_4 = 0$
        \item Model 2: $a_1 = 1,\quad a_2 = 0,\quad a_3 = a^3,\quad a_4 = 1 - a^2$
        \item Model 3: $a_1 = a^3,\quad a_2 = 1 - a^2,\quad a_3 = a^2,\quad a_4 = 1 - a$
        \item Model 4: $a_1 = 1,\quad a_2 = 0,\quad a_3 = \frac{a^2}{a^2 + 1},\quad a_4 = \frac{1}{a^2 + 1}$
        \item Model 5: $a_1 = \frac{1}{4} - a^2,\quad a_2 = a^2,\quad a_3 = a,\quad a_4 = \max\left(0, \frac{7}{8} - \frac{a}{2} - \frac{a^2}{2}\right)$
\end{enumerate}
    
These parameter settings were chosen to reflect a wide range of performative behaviors, including nonlinear dependences on the shift parameter $a$. The results are shown in \cref{fig:combined_performative_results}. In each case, a weak performative effect renders both the ERM and PERM methods similar in accuracy. As the strength of the performativity increases, the PERM method consistently outperforms ERM, demonstrating its robustness to performative shifts. The results also indicate that the choice of model parameters can significantly influence the performance of both methods, with some configurations leading to more pronounced differences than others.
    
    \begin{figure}
        \centering
        \includegraphics[width=\textwidth]{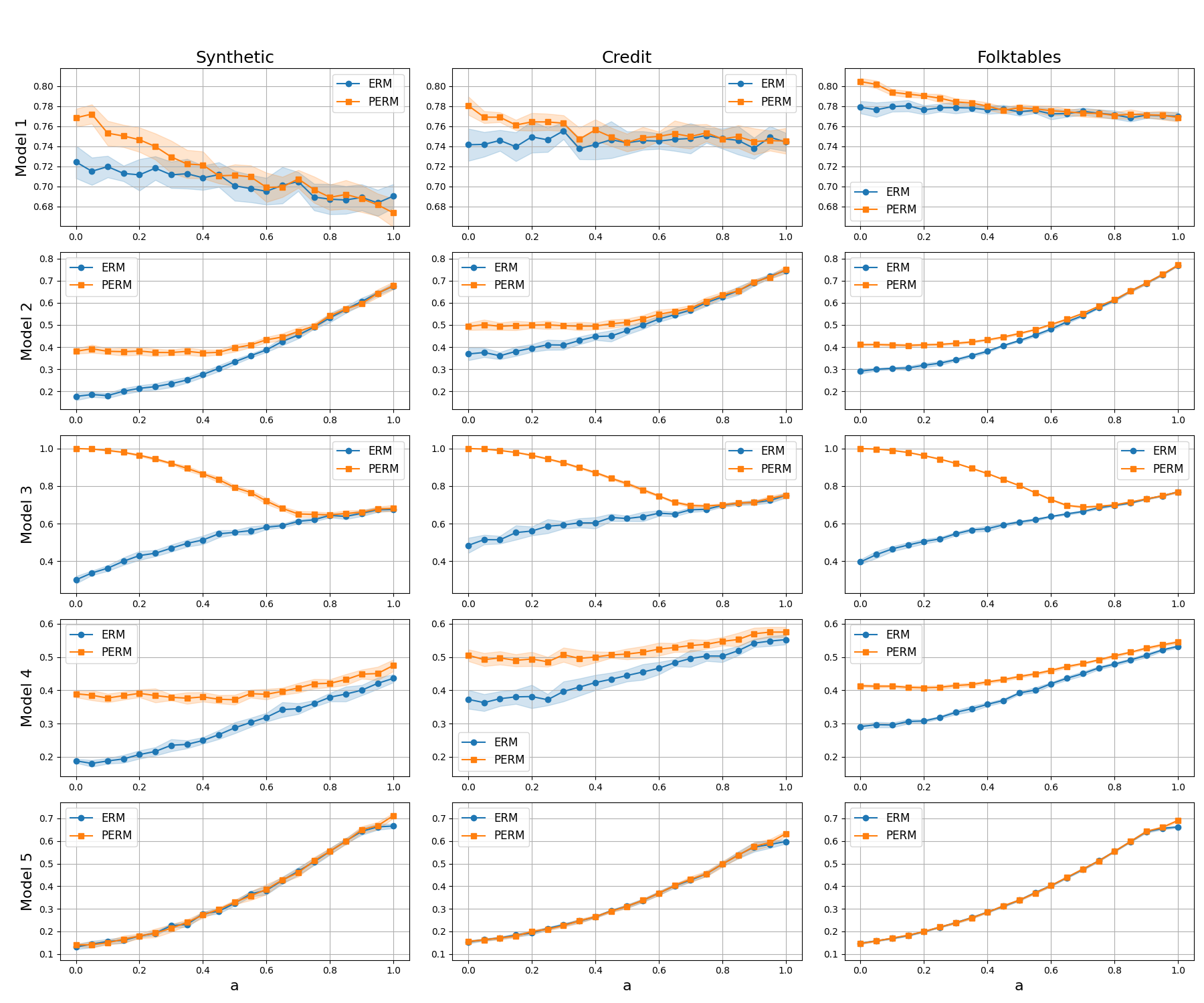}
        \caption{Experiments on the models 1 - 5 for the three datasets: Synthetic, Credit Score, and Folktables.}
       \label{fig:combined_performative_results}
    \end{figure}

\section{Results for Finite Hypothesis}
\label{appendix: B}
\begin{theorem}
    \label{theorem: finite_hypothesis}
    Let $\mathcal{H}$ be a class of hypothesis functions with $|\mathcal{H}| < \infty$ and let $\R(h)$ and $\R_n(h)$ be defined as in \cref{lemma: performative_risk}. Then for any $\delta > 0$ with probability at least $1-\delta$ we have that
    \begin{equation*}
        \forall h \in \mathcal{H}, \quad |\R(h) - \R_n(h)| \le \sqrt{\frac{ A_{diff}^2 (\log{2|\mathcal{H}|} + \log{\frac{1}{\delta}} )}{2n}},
    \end{equation*}
    where 
    \begin{equation*}
        A_{diff} := 2 (|\alpha_1| + |\alpha_2| + |\alpha_3|).
    \end{equation*}
\end{theorem}

\begin{proof}
    First we show that we can apply Hoeffding's inequality to concentrate the performative empirical risk to the performative risk. The random variables $\{ f_h(X_i, Y_i) \}_{i=1}^n$ are i.i.d. with 
    \begin{equation*}
        A_{min} \le f_h(X_i, Y_i) \le A_{max},
    \end{equation*}
    where 
    \begin{equation*}
        A_{max} := |\alpha_1| + |\alpha_2| + |\alpha_3| + \alpha_4, \quad A_{min} := \alpha_4 - (|\alpha_1| + |\alpha_2| + |\alpha_3|)
    \end{equation*}
    which follows directly from the definition of the class $\mathcal{F}_{\mathcal{H}}$. Therefore, we apply Hoeffding's inequality, which gives
    \begin{equation*}
        \mathbb{P}\bigg[  \bigg|  \frac{1}{n}\sum_{i=1}^n f_h(X_i, Y_i) - \mathbb{E}[f_h(X_1, Y_1)] \bigg| > \epsilon \bigg] \le 2 \exp\bigg(\frac{-2n\epsilon^2 }{A_{diff}^2}\bigg),
    \end{equation*}
    where 
    \begin{equation*}
        A_{diff} = A_{max} - A_{min} = 2 (|\alpha_1| + |\alpha_2| + |\alpha_3|).
    \end{equation*}
    As we know from \cref{lemma: performative_risk}, 
    \begin{equation*}
        \mathbb{E}[f_h(X_1, Y_1)] = \R(h),
    \end{equation*}
    and so
    \begin{align*}
        \mathbb{P}[\exists h \in \mathcal{H} \, : \,  | \R_n(h) - \tilde{R}(h) |  > \epsilon] &= \mathbb{P}[( | \R_n(h_1) - \tilde{R}(h_1) |  > \epsilon) \vee \dots ( | \R_n(h_{ | \mathcal{H} | }) - \tilde{R}(h_{ | \mathcal{H} | }) |  > \epsilon)]\\
        &\le \sum_{h \in \mathcal{H}} \mathbb{P}[ | \R_n(h) - \tilde{R}(h) | >\epsilon]\\
        &\le 2 | \mathcal{H} | \exp\bigg(\frac{-2n\epsilon^2 }{A_{diff}^2}\bigg).
    \end{align*}
    This means that with probability at least $1-\delta$, 
    \begin{equation*}
        \forall h \in \mathcal{H}, \quad |\R(h) - \R_n(h)| \le \sqrt{\frac{ A_{diff}^2 (\log{2|\mathcal{H}|} + \log{\frac{1}{\delta}} )}{2n}}.
    \end{equation*}

\end{proof}

\begin{corollary}
    In the setting of \cref{theorem: finite_hypothesis}, let $h_S^\ast \in \argmin_{h \in \mathcal{H}} \R_n(h)$. Then with probability at least $1-\delta$
    \begin{equation*}
        \R(h_S^\ast) \le \min_{h \in \mathcal{H}} \R(h) + \sqrt{\frac{ 2 A_{diff}^2 (\log{2|\mathcal{H}|} + \log{\frac{1}{\delta}} )}{n}}.
    \end{equation*}
\end{corollary}
\begin{proof}
    The proof is essentially the same as in \cite{mohri2018foundations} and we give it here for completeness only. Using \cref{theorem: finite_hypothesis}, we have that for every $h \in \mathcal{H}$,
\begin{align*}
    \R(h_S^\ast) &\le \R_n(h_S^\ast) + \sqrt{\frac{ A_{diff}^2 (\log{2|\mathcal{H}|} + \log{\frac{1}{\delta}} )}{2n}}\\
    &\le \R_n(h) + \sqrt{\frac{ A_{diff}^2 (\log{2|\mathcal{H}|} + \log{\frac{1}{\delta}} )}{2n}} \\
    &\le \R(h) + 2\sqrt{\frac{ A_{diff}^2 (\log{2|\mathcal{H}|} + \log{\frac{1}{\delta}} )}{2n}}.
\end{align*}
\end{proof}

\section{Proofs}
\label{appendix: C}

\subsection{Proof of \cref{lemma: performative_risk}}
\begin{lemma*}
    Let $S = \{ (X_i, Y_i) \}_{i=1}^n$ be a set of $n$ i.i.d samples from $\mathcal{D}$ and let $\tilde{\mathcal{D}}$ be the distribution obtained from $\mathcal{D}$ after the performative linear shift described in \cref{sec: model}. For a hypothesis $h \in \mathcal{H}$, the performative linear risk is given by
    Let \(S = \{ (X_i, Y_i) \}_{i=1}^n\) be a set of \(n\) i.i.d.\ samples from \(\mathcal{D}\), and let \(h \in \mathcal{H}\) be a hypothesis. Suppose \(\tilde{\mathcal{D}}(h)\) is the distribution obtained from \(\mathcal{D}\) after a performative linear shift as defined in \cref{eq: distribution_map}. Then the performative risk \(\R(h)\) is given by
    \begin{align*}
        \R(h)= (1-a_2) \p[h(X) = 1]
        - a_1 \p[Y=1, h(X)=1]
        + a_3 \p[Y=1, h(X)=-1]\\
        + a_4 \p[ h(X)=-1]. 
    \end{align*}
    Moreover, an unbiased estimator of \(\R(h)\) is the performative empirical risk \vspace{-0.15cm}
    \begin{equation*}
        {\R}_n(h) := \frac{1}{n} \sum_{i=1}^n \big(\alpha_1 h(X_i) + \alpha_2 Y_i + \alpha_3 Y_i h(X_i) + \alpha_4 \big), 
    \end{equation*}
    where \(\alpha_1 = \frac{2 - a_1 - 2a_2 -a_3 - 2a_4}{4}, \ \alpha_2 = \frac{a_3 - a_1}{4}, \
        \alpha_3 = \frac{-a_1 - a_3}{4}, \ \alpha_4 = \frac{2 - a_1 - 2a_2 + a_3 + 2a_4}{4}.\)
\end{lemma*}
\begin{proof}
    Let $F_X(x)$ denote the distribution function of $X$. Using the law of total probability and the linear relationship between the measures $\tilde{\p}_h$ and $\p$ we have 
    \begin{align*}
        \tilde{\p}_h[Y \ne h(X)] &= \int_{\{x \in \mathcal{X}\}} \tilde{\p}_h[Y \ne h(X) | X = x] dF_X(x)\\
        &= \int_{\{x \in \mathcal{X} : h(x) = 1\}} \tilde{\p}_h[Y = -1 | X = x] dF_X(x) +  \int_{\{x \in \mathcal{X} : h(x) = -1\}} \tilde{\p}_h[Y = 1 | X = x] dF_X(x)\\
        &=  \int_{\{x \in \mathcal{X} : h(x) = 1\}} (1 - a_1 \p[Y=1|X=x] - a_2) dF_X(x)\\
        &+  \int_{\{x \in \mathcal{X} : h(x) = -1\}} (a_3 \p[Y=1|X=x] + a_4) dF_X(x)\\
        &=(1 - a_2)\int_{\{x \in \mathcal{X} : h(x) = 1\}} dF_X(x)  - a_1 \int_{\{x \in \mathcal{X} : h(x) = 1\}} \p[Y=1 | X=x] dF_X(x) \\
        & + a_3 \int_{\{x \in \mathcal{X} : h(x) = -1\}} \p[Y=1|X=x] dF_X(x)  + a_4 \int_{\{x \in \mathcal{X} : h(x) = -1\}} dF_X(x)\\
        &= (1-a_2) \p[h(X) = 1] - a_1 \p[Y=1, h(X)=1] + a_3 \p[Y=1, h(X)=-1] + a_4 \p[ h(X)=-1]
    \end{align*}
    The performative risk is now written in terms of the measure $\p$, and therefore it can be approximated only using the samples $S = \{ (X_1, Y_1), \dots, (X_n , Y_n) \}$ from $\mathcal{D}$. An unbiased estimator of $\R(h)$ is given by
    \begin{equation*}
        \R_n(h) := \frac{1}{n} \sum_{i=1}^n (1 - a_2) \mathbf{1}_{\{ h(X_i) = 1\}} - a_1 \mathbf{1}_{\{ Y_i = 1, h(X_i) = 1 \}} + a_3 \mathbf{1}_{\{ Y_i = 1, h(X_i)=-1 \}} + a_4 \mathbf{1}_{\{ h(X_i)=-1 \}}.
    \end{equation*}
    Since $Y, h(X) \in \{ -1, 1\}$ for all $(X, Y)$, this can be written as 
    \begin{align*}
        \R_n(h) &:= \frac{1}{n} \sum_{i=1}^n \frac{1 - a_2}{2} (1 + h(X_i)) - \frac{a_1}{4} (1 + Y_i)(1 + h(X_i)) + \frac{a_3}{4} (1 + Y_i)(1 - h(X_i)) + \frac{a_4}{2} (1-h(X_i))\\
        &= \frac{2 - a_1 - 2a_2 -a_3 - 2a_4}{4} h(X_i) + \frac{a_3 - a_1}{4} Y_i + \frac{-a_1 - a_3}{4} Y_i h(X_i) + \frac{2 - a_1 - 2a_2 + a_3 + 2a_4}{4},
    \end{align*}
    which finishes the proof.
\end{proof}

\subsection{Proof of \cref{theorem: infinite_hypothesis}}
\begin{theorem*}
    Let $\mathcal{H}$ be a hypothesis class, $S = \{(x_i, y_i)\}_{i=1}^n$ be $n$ i.i.d. samples from $\mathcal{D}$, and let $\R_n(h)$ be defined as in \cref{eq: empirical_risk}. Then, given a distribution map $\tilde{D}(\cdot)$ defined as in \cref{eq: distribution_map}, for any $\delta > 0$ with probability at least $1-\delta$ we have that for all $h \in \mathcal{H}$ \vspace{-0.15cm}
    \begin{align*}
       |\R(h) - \R_n(h)| \le  2 (|\alpha_1| + |\alpha_3|) \mathfrak{R}_n(\mathcal{H}) 
        + 2((1 + \sqrt{2})|\alpha_1| + |\alpha_2| + (1 + \sqrt{2})|\alpha_3|)\sqrt{\frac{\ln{\frac{4}{\delta}}}{2n}}.
    \end{align*} 
\end{theorem*}
\begin{proof}
    The proof of \cref{theorem: infinite_hypothesis} builds upon the proof of Theorem 5 in \cite{Oliver}, which establishes upper bounds for the non-performative case. However, we introduce several crucial modifications to account for the performative effect.
    \begin{itemize}
        \item Step 1. 
        We start by showing that McDiarmid's inequality can be applied to $\sup_{h \in \mathcal{H}} \{\R(h) - \R_n(h)\}$. Specifically, we denote by $\R_n'(h)$ the performative empirical risk obtained by modifying one element, i.e. $(X_i, Y_i)$ is replaced by $(X_i', Y_i')$. Then
        \begin{align*}
            | \sup_{h \in \mathcal{H}} \{\R(h) - \R_n(h)\} - \sup_{h \in \mathcal{H}} \{\R(h) - \R_n'(h)\} | 
            &\le \sup_{h \in \mathcal{H}} | \R_n'(h) - \R_n(h) |\\
            &= \sup_{h \in \mathcal{H}} \frac{1}{n} | f_h(X_i', Y_i') - f_h(X_i, Y_i) | \\
            &\le \frac{2(|\alpha_1| + |\alpha_2| + |\alpha_3|)}{n}
        \end{align*}
        where the last inequality comes from the definition of the function $f_h \in \mathcal{F}_{\mathcal{H}}$. Thus we can apply McDiarmid's inequality with $c={2(|\alpha_1| + |\alpha_2| + |\alpha_3|)}/{n}$, giving us that
        \begin{equation*}
            \p\bigg[  \sup_{h \in \mathcal{H}} \{\R(h) - \R_n(h)\} - \E \bigg[\sup_{h \in \mathcal{H}} \{\R(h) - \R_n(h)\}  \bigg]   \ge \epsilon \bigg] \le  \exp{\bigg(\frac{- n \epsilon^2}{2(|\alpha_1| + |\alpha_2| + |\alpha_3|)^2} \bigg)}
        \end{equation*}
        or equivalently, with probability at least $1-\delta/4$
        \begin{equation}
            \sup_{h \in \mathcal{H}} \{\R(h) - \R_n(h)\} \le \E\bigg[\sup_{h \in \mathcal{H}} \{\R(h) - \R_n(h)\}\bigg] + \sqrt{\frac{2(|\alpha_1| + |\alpha_2| + |\alpha_3|)^2\ln{\frac{4}{\delta}}}{n}}. \label{eq: bound 1}
        \end{equation}
        \item Step 2.
        Next, we use symmetrization to relate the expectation $\E\bigg[\sup_{h \in \mathcal{H}} \{\R(h) - \R_n(h)\}\bigg]$ to the Rademacher average. 

        In the following,  we introduce a ghost sample $S' = \{Z_1', \dots, Z_n' \}$ and use the fact that $\R(h) = \E_{S'}[\R_n'(h)]$. Also, we denote by $\E_S$ the expectation w.r.t. the sample $S = \{Z_1, \dots, Z_n \}$ and by $\E$ the full expectation. We also use Jensen's inequality for the supremum function and that $f_h(Z_i') - f_h(Z_i)$ and $\sigma_i (f_h(Z_i') - f_h(Z_i))$ have the same distribution for all $h \in \mathcal{H}$.
        \begin{align*}
            \E\left[\sup_{h \in \mathcal{H}} \left\{\R(h) - \R_n(h)\right\}\right] 
            &= \E_S\left[\sup_{h \in \mathcal{H}} \left\{\E_{S'}\left[\R_n'(h)\right] - \R_n(h)\right\}\right] \\
            &= \E_S\left[\sup_{h \in \mathcal{H}} \left\{\E_{S'}\left[\R_n'(h) - \R_n(h)\right]\right\}\right] \\
            &\leq \E_S\left[\E_{S'}\left[\sup_{h \in \mathcal{H}} \left\{\R_n'(h) - \R_n(h)\right\}\right]\right] \\
            &= \E\left[\sup_{h \in \mathcal{H}} \left\{\R_n'(h) - \R_n(h)\right\}\right] \\
            &= \E\left[\sup_{h \in \mathcal{H}} \left\{\frac{1}{n} \sum_{i=1}^n \sigma_i \left(f_h(Z_i') - f_h(Z_i)\right)\right\}\right] \\
            &\leq \E\left[\sup_{h \in \mathcal{H}} \left\{\frac{1}{n} \sum_{i=1}^n \sigma_i f_h(Z_i')\right\}\right] 
                + \E\left[\sup_{h \in \mathcal{H}} \left\{\frac{1}{n} \sum_{i=1}^n -\sigma_i f_h(Z_i)\right\}\right] \\
            &= 2 \E\left[\sup_{h \in \mathcal{H}} \left\{\frac{1}{n} \sum_{i=1}^n \sigma_i f_h(Z_i)\right\}\right] \\
            &= 2 \E\left[\sup_{f_h \in \mathcal{F}_{\mathcal{H}}} \left\{\frac{1}{n} \sum_{i=1}^n \sigma_i f_h(Z_i)\right\}\right] \\
            &= 2 \mathfrak{R}(\mathcal{F}_{\mathcal{H}})
        \end{align*} 
        Note that the above is almost identical to the symmetrization lemma (e.g. from \cite{Oliver}), with the small change of supremum from $\mathcal{H}$ to $\mathcal{F}_{\mathcal{H}}$ on the penultimate equality, which holds since for each function $h$ we have only one value for $f_h$. 

        Combining this with equation \eqref{eq: bound 1}, we get that with probability at least $1-\delta/4$
        \begin{equation}
            \sup_{h \in \mathcal{H}} \{\R(h) - \R_n(h)\} \le 2 \mathfrak{R}(\mathcal{F}_{\mathcal{H}}) + \sqrt{\frac{2(|\alpha_1| + |\alpha_2| + |\alpha_3|)^2\ln{\frac{4}{\delta}}}{n}}. \label{eq: bound 2}
        \end{equation}
        \item Step 3.
        Now we need to relate the Rademacher averages between the classes $\mathcal{F}_{\mathcal{H}}$ and $\mathcal{H}$. We have 
        \begin{align*}
            \mathfrak{R}(\mathcal{F}_{\mathcal{H}}) &= \E  \bigg[\sup_{f_h\in \mathcal{F}_{\mathcal{H}}} \frac{1}{n} \sum_{i=1}^n \sigma_i f_h(Z_i) \bigg]\\
            &= \E  \bigg[\sup_{h \in \mathcal{H}} \frac{1}{n} \sum_{i=1}^n \sigma_i (\alpha_1 h(X_i) + \alpha_2 Y_i+ \alpha_3 Y_ih(X_i) + \alpha_4) \bigg]\\
            &= \E \bigg[ \sup_h \bigg\{ \frac{1}{n} \sum_{i=1}^n h(X_i)( \alpha_1 \sigma_i + \alpha_3 Y_i \sigma_i  ) \bigg\} + \frac{1}{n} \sum_{i=1}^n \alpha_2 Y_i \sigma_i + \alpha_4 \sigma_i\bigg]\\
            &= \E \bigg[  \sup_{h \in \mathcal{H}} \frac{1}{n} \sum_{i=1}^n  h(X_i) (\alpha_1 \sigma_i + \alpha_3 \sigma_i') \bigg]\\
            &\le \E \bigg[  \sup_{h \in \mathcal{H}} \frac{1}{n} \sum_{i=1}^n  \alpha_1 h(X_i) \sigma_i \bigg] + \E \bigg[  \sup_{h \in \mathcal{H}} \frac{1}{n} \sum_{i=1}^n  \alpha_3 h(X_i) \sigma_i' \bigg]\\
            &= (|\alpha_1| + |\alpha_3|) \mathfrak{R}(\mathcal{H})
        \end{align*}
        where we have used that $\sigma_i Y_i$ has the same distribution as a Rademacher random variable, $\sigma_i'$.

        Combining this with equation \eqref{eq: bound 2}, we get that with probability at least $1-\delta/4$
        \begin{equation}
            \sup_{h \in \mathcal{H}} \{\R(h) - \R_n(h)\} \le 2 (|\alpha_1| + |\alpha_3|) \mathfrak{R}(\mathcal{H}) + \sqrt{\frac{2(|\alpha_1| + |\alpha_2| + |\alpha_3|)^2\ln{\frac{4}{\delta}}}{n}}. \label{eq: bound 3}
        \end{equation}
        \item Step 4.
        Further, we need to relate the Rademacher average $\mathfrak{R}(\mathcal{H})$ to the conditional Rademacher average $\mathfrak{R}_n(\mathcal{H})$. This part of the proof remains the same as in \cite{mohri2018foundations,Oliver}. We apply McDiarmid's inequality to $\mathfrak{R}_n(\mathcal{H})$ with $\E[\mathfrak{R}_n(\mathcal{H})] = \mathfrak{R}(\mathcal{H})$, which holds with $c= 2 / n$. Therefore, with probability at least $1-\delta/4$ 
        \begin{equation*}
            \mathfrak{R}(\mathcal{H}) \le \mathfrak{R}_n(\mathcal{H}) + \sqrt{\frac{\log{\frac{4}{\delta}}}{n}}.
        \end{equation*}
        Combining this with equation \eqref{eq: bound 3}, we get that with probability at least $1-\delta/2$ 
        \begin{equation*}
            \sup_{h \in \mathcal{H}} \{\R(h) - \R_n(h)\} \le 2 (|\alpha_1| + |\alpha_3|) \mathfrak{R}_n(\mathcal{H}) + 2((1 + \sqrt{2})|\alpha_1| + |\alpha_2| + (1 + \sqrt{2})|\alpha_3|)\sqrt{\frac{\ln{\frac{4}{\delta}}}{2n}}.
        \end{equation*}
        \item Step 5.
        Lastly, bounding $\sup_{h \in \mathcal{H}} \{ \R_n(h) - \R(h) \}$ with the same upper bound with probability $1-\delta/2$ follows in a similar way. The final result of the theorem then follows by applying the union bound.
    \end{itemize}  
\end{proof}

\subsection{Proof of \cref{corollary: pac_learnability_infinite}}
\begin{corollary*}
    In the setting of \cref{theorem: infinite_hypothesis}, let $h_S^\ast \in \argmin_{h \in \mathcal{H}} \R_n(h)$. Then with probability at least $1-\delta$ \vspace{-0.25cm}
    \begin{align*}
        \R(h_S^\ast) \le \min_{h \in \mathcal{H}} \R(h) + 4 (|\alpha_1| + |\alpha_3|) \mathfrak{R}_n(\mathcal{H})
         + 4((1 + \sqrt{2})|\alpha_1| + |\alpha_2| + (1 + \sqrt{2})|\alpha_3|)\sqrt{\frac{\ln{\frac{4}{\delta}}}{2n}}.
    \end{align*}
    In particular, if $\mathcal{H}$ is agnostic PAC-learnable, then it is performatively agnostic PAC-learnable w.r.t. the distribution map $\tilde{\mathcal{D}}(\cdot)$. 
\end{corollary*}
\begin{proof}
    The proof is essentially the same as in \cite{mohri2018foundations} and we give it here for completeness only. Using \cref{theorem: infinite_hypothesis}, we have that for every $h \in \mathcal{H}$, with probability at least $1-\delta$
\begin{align*}
    \R(h_S^\ast) &\le \R_n(h_S^\ast) + 2 (|\alpha_1| + |\alpha_3|) \mathfrak{R}_n(\mathcal{H}) + 2((1 + \sqrt{2})|\alpha_1| + |\alpha_2| + (1 + \sqrt{2})|\alpha_3|)\sqrt{\frac{\ln{\frac{4}{\delta}}}{2n}}\\
    &\le \R_n(h) + 2 (|\alpha_1| + |\alpha_3|) \mathfrak{R}_n(\mathcal{H}) + 2((1 + \sqrt{2})|\alpha_1| + |\alpha_2| + (1 + \sqrt{2})|\alpha_3|)\sqrt{\frac{\ln{\frac{4}{\delta}}}{2n}}\\
    &\le \R(h) + 4 (|\alpha_1| + |\alpha_3|) \mathfrak{R}_n(\mathcal{H}) + 4((1 + \sqrt{2})|\alpha_1| + |\alpha_2| + (1 + \sqrt{2})|\alpha_3|)\sqrt{\frac{\ln{\frac{4}{\delta}}}{2n}}.
\end{align*}
\end{proof}

\subsection{Proof of \cref{theorem: imperfect_information}}
\begin{theorem*}
    Let $\mathcal{H}$ be a hypothesis class, $S = \{(x_i, y_i)\}_{i=1}^n$ be $n$ i.i.d. samples from $\mathcal{D}$, and let $\R(h)$ and $\R_n(h)$ be defined as in \cref{eq: risk} and \cref{eq: empirical_risk}. For the map $\tilde{\mathcal{D}}(\cdot)$ defined as above, if $h_S^\ast \in \argmin_{h \in \mathcal{H}} \overline{\R}_n(h)$, then with probability at least $1-\delta$
    \begin{align*}
        \R(h_S^\ast) \le \min_{h \in \mathcal{H}} \R(h) + 2\epsilon + 2 A \mathfrak{R}_n(\mathcal{H})
         + 2B \sqrt{\frac{\ln{\frac{4}{\delta}}}{2n}},
    \end{align*}
    where $A = 2 (|\alpha_1| + |\alpha_3|)$ and $B = 2((1 + \sqrt{2})|\alpha_1| + |\alpha_2| + (1 + \sqrt{2})|\alpha_3|)$.
\end{theorem*}

\begin{proof}
    First we relate $\overline{\R}_n(h)$ to $\R_n(h)$: 
    \begin{align*}
        |\overline{\R}_n(h) - \R_n(h)| &= \bigg| \frac{1}{n} \sum_{i=1}^n \overline{f}_h(X_i, Y_i) - f_h(X_i, Y_i) \bigg|\\
        &\le \sum_{i=1}^n \frac{|\overline{f}_h(X_i, Y_i) - f_h(X_i, Y_i)|}{n}\\
        &= \sum_{i=1}^n \frac{|(\overline{\alpha}_1 - \alpha_1) h(X_i) + (\overline{\alpha}_2 - \alpha_2) Y_i + (\overline{\alpha}_3 - \alpha_3) Y_i h(X_i) + (\overline{\alpha}_4 - \alpha_4) |}{n}\\
        &\le \sum_{i=1}^n \frac{|\overline{\alpha}_1 - \alpha_1| + |\overline{\alpha}_2 - \alpha_2| + |\overline{\alpha}_3 - \alpha_3| + |\overline{\alpha}_4 - \alpha_4 |}{n}\\
        &= |\overline{\alpha}_1 - \alpha_1| + |\overline{\alpha}_2 - \alpha_2| + |\overline{\alpha}_3 - \alpha_3| + |\overline{\alpha}_4 - \alpha_4 |,
    \end{align*}
    where 
    \begin{equation*}
        \overline{f}_n := \overline{\alpha}_1 h(X_i) + \overline{\alpha}_2 Y_i + \overline{\alpha}_3 Y_i h(X_i) + \overline{\alpha}_4 .
    \end{equation*}
    It is easy to check that 
    \begin{align*}
        |\overline{\alpha}_1 - \alpha_1| \le \frac{\epsilon_1 + 2\epsilon_2 + \epsilon_3 + 2\epsilon_4}{4}, |\overline{\alpha}_2 - \alpha_2| \le \frac{\epsilon_1 + \epsilon_3}{4} , |\overline{\alpha}_3 - \alpha_3| \le  \frac{\epsilon_1 + \epsilon_3}{4}, |\overline{\alpha}_4 - \alpha_4| \le \frac{\epsilon_1 + 2\epsilon_2 + \epsilon_3 + 2\epsilon_4}{4},
    \end{align*}
    which gives us that 
    \begin{equation*}
        |\overline{\R}_n(h) - \R_n(h)| \le \sum_{i=1}^4 \epsilon_i.
    \end{equation*}
    Note that if $\overline{a}_i$ are all taken to be the midpoints of the intervals $I_i$, then this can be reduced to 
    \begin{equation*}
        |\overline{\R}_n(h) - \R_n(h)| \le \frac12 \sum_{i=1}^4 \epsilon_i.
    \end{equation*}
    Now we have that 
    \begin{align*}
        \sup_{h \in \mathcal{H}} \{ | \R(h) - \overline{\R}_n(h) |\} &=  \sup_{h \in \mathcal{H}} \{ | \R(h) - \R_n(h) + \R_n(h) - \overline{\R}_n(h) | \}\\
        &\le \sup_{h \in \mathcal{H}} \{ | \R(h) - \R_n(h)| \} + \sup_{h \in \mathcal{H}} \{ | \R_n(h) - \overline{\R}_n(h) | \}\\
        &\le \sup_{h \in \mathcal{H}} \{ |\R(h) - \R_n(h) | \} + \sum_{i=1}^4 \epsilon_i.
    \end{align*}
    Finally,  we have that for every $h \in \mathcal{H}$, with probability at least $1-\delta$
    \begin{align*}
        \R(h_S^\ast) &\le\overline{\R}_n(h_S^\ast) + \sum_{i=1}^4 \epsilon_i + 2 (|\alpha_1| + |\alpha_3|) \mathfrak{R}_n(\mathcal{H}) + 2((1 + \sqrt{2})|\alpha_1| + |\alpha_2| + (1 + \sqrt{2})|\alpha_3|)\sqrt{\frac{\ln{\frac{4}{\delta}}}{2n}}\\
        &\le\overline{\R}_n(h) + \sum_{i=1}^4 \epsilon_i+ 2 (|\alpha_1| + |\alpha_3|) \mathfrak{R}_n(\mathcal{H}) + 2((1 + \sqrt{2})|\alpha_1| + |\alpha_2| + (1 + \sqrt{2})|\alpha_3|)\sqrt{\frac{\ln{\frac{4}{\delta}}}{2n}}\\
        &\le \R(h)+ 2\sum_{i=1}^4 \epsilon_i + 4 (|\alpha_1| + |\alpha_3|) \mathfrak{R}_n(\mathcal{H}) + 4((1 + \sqrt{2})|\alpha_1| + |\alpha_2| + (1 + \sqrt{2})|\alpha_3|)\sqrt{\frac{\ln{\frac{4}{\delta}}}{2n}}.
    \end{align*}
\end{proof}

\subsection{Proof of \cref{proposition: no_free_lunch_lower_bound}}
\begin{proposition*}{(No Free Lunch Lower Bound)} 
Consider an arbitrary input-output space $\mathcal{X} \times \mathcal{Y}$ and a hypothesis space $\mathcal{H}$ containing at least two distinct functions. For any learning algorithm $\mathcal{A}$, there exists a distribution $\mathcal{D}$ and a set of true parameters ${a}_i$, that lie within known intervals ${I}_i$, of total length $\epsilon = \sum_{i=1}^{d} \epsilon_i$, such that with probability at least $1/2$ over the random choice of the training set $S \sim \mathcal{D}^{ n}$ and the possible randomness of the algorithm, the excess risk of the algorithm's output ${h}_S$ is large: 
\begin{equation*}
    \R({h}_S) - \min_{h \in \mathcal{H}} \R(h) \geq \frac{\epsilon}{4}.
\end{equation*}
\end{proposition*}

\begin{proof}
We use the following argument, constructing two "worlds" that are indistinguishable from the training data but require different optimal decisions. For simplicity, we take the entire uncertainty to be in the parameter $a_1$, so $\epsilon_1 = \epsilon$ and $\epsilon_2 = \epsilon_3 = \epsilon_4 = 0$. Let the interval for $a_1$ be $I_1 = [\overline{a}_1 - \epsilon/2, \overline{a}_1 + \epsilon/2]$. 

Let $x_0 \in \mathcal{X}$ be such that some functions $h_1, h_2 \in \mathcal{H}$ differ on $x_0$ (such $x_0$ exists as $|\mathcal{H}| \geq 2$). We define a distribution $\mathcal{D}$ over the space $\mathcal{X} \times \mathcal{Y}$, such that its entire probability mass is concentrated on $x_0$. Further, let the initial label distribution be $\p[Y = 1 | X = x_0] = 1/2$. The training data $S$ thus consists of $n$ i.i.d. samples of the label $Y \in \{-1, 1\}$, since the feature is $x_0$ with probability 1.

We can partition the hypothesis class $\mathcal{H}$ into two (non-empty) disjoint subsets: $\mathcal{H}_1 = \{h \in \mathcal{H} \mid h(x_0) = 1\}$ and $\mathcal{H}_{-1} = \{h \in \mathcal{H} \mid h(x_0) = -1\}$. We define two "worlds" based on two possible true values for $a_1$:
\begin{itemize}
    \item World 1 ($V_1$): The true parameter is $a_1^{(1)} = \overline{a}_1 - \epsilon/2$.
    \item World 2 ($V_2$): The true parameter is $a_1^{(2)} = \overline{a}_1 + \epsilon/2$.
\end{itemize}

From \cref{lemma: performative_risk}, we can calculate the performative risk in each case. For any hypothesis $h_{-1} \in \mathcal{H}_{-1}$, the risk is $\R(h_{-1}) = a_4 + a_3 /2$. For any $h_1 \in \mathcal{H}_1$, the risk is $\R(h_1; a_1) = (1 - a_2) - a_1/2$. We fix the parameters $a_2, a_3, a_4$ such that the optimal decision flips between the two worlds by setting:
\[
\R(h_{-1}) = (1 - a_2) - \overline{a}_1/2.
\]

With this construction, we have $\R(h_1; a_1^{(1)}) < \R(h_{-1}) < \R(h_1; a_1^{(2)})$. This means in World 1, the optimal hypotheses are in $\mathcal{H}_1$, and in World 2, the optimal hypotheses are in $\mathcal{H}_{-1}$. The excess risk for choosing the suboptimal hypothesis class in each world is $\epsilon/4$, since:
\begin{itemize}
    \item In $V_1$: $\R(h_{-1}) - \R(h_1; a_1^{(1)}) = \frac{a_1^{(1)}}{2} - \frac{\overline{a}_1}{2} = \frac{\epsilon}{4}$.
    \item In $V_2$: $\R(h_1; a_1^{(2)}) - \R(h_{-1}) = \frac{\overline{a}_1}{2} - \frac{a_1^{(2)}}{2} = \frac{\epsilon}{4}$.
\end{itemize}

Now, consider any learning algorithm $\mathcal{A}$, which may be randomized. The distribution of the training set $S$ is identical in both worlds and thus the algorithm has no information from the data to distinguish between $V_1$ and $V_2$. Let $p = \p_{S, \mathcal{A}}(\mathcal{A}(S) \in \mathcal{H}_1)$ be the probability (over the random draw of $S$ and any internal randomness of $\mathcal{A}$) that the algorithm outputs a hypothesis from $\mathcal{H}_1$. Since the initial data distribution is the same in both worlds, this probability $p$ is the same for both $V_1$ and $V_2$.

We now show that at least one of these worlds constitutes a "bad" scenario for the algorithm $\mathcal{A}$. We consider two cases based on the value of $p$:
\begin{itemize}
    \item Case 1: $p \geq 1/2$. In this case, we choose the parameters of World 2 ($V_2$) as our hard instance. In $V_2$, the optimal choice is a hypothesis from $\mathcal{H}_{-1}$. The probability that the algorithm makes a suboptimal choice (i.e., outputs hypotheses from $\mathcal{H}_1$) is $p$, which is at least $1/2$.
    \item Case 2: $p < 1/2$. In this case, we choose the parameters of World 1 ($V_1$) as our hard instance. In $V_1$, the optimal choice is a hypothesis from $\mathcal{H}_1$. The probability that the algorithm makes a suboptimal choice (i.e., outputs a hypothesis from $\mathcal{H}_{-1}$) is $1 - p$, which is greater than $1/2$.
\end{itemize}

In either case, we have shown the existence of a set of parameters for which the algorithm makes a suboptimal choice with probability at least $1/2$. Since the excess risk of any suboptimal choice is $\epsilon/4$, the statement of the proposition follows.
\end{proof}

\subsection{Proofs of \cref{theorem: general_shift}}

\begin{theorem*}
    Let $\mathcal{H}$ be a hypothesis class, $S = \{(x_i, y_i)\}_{i=1}^n$ be $n$ i.i.d. samples from $\mathcal{D}$, and let $\R(h)$ and $\R_n(h)$ be defined as \cref{eq: risk} and \cref{eq: rn_empirical_risk}. If \cref{assumption: general_shift} are satisfied, taking $h_S^\ast \in \argmin_{h \in \mathcal{H}} \R_n(h)$, for any $\delta > 0$ with probability at least $1-\delta$ we have that 
    \begin{align*}
        \R(h_S^\ast) \le \min_{h \in \mathcal{H}} \R(h) + 2M \mathfrak{R}_n(\mathcal{H}) + (4 + 2\sqrt{2})M\sqrt{\frac{\ln{\frac{4}{\delta}}}{2n}}.
    \end{align*}
    In particular, if $\mathcal{H}$ is agnostic PAC-learnable, then it is performatively agnostic PAC-learnable w.r.t. the distribution map $\tilde{\mathcal{D}}(\cdot)$. 
\end{theorem*}
\begin{proof}
The proof largely follows the steps of \cref{theorem: infinite_hypothesis}, with only minor differences in steps $1$ and $3$, which we highlight here. 
\begin{itemize}
    \item Step 1.
    We apply once again McDiarmid's inequality to $\sup_{h \in \mathcal{H}} \{\R(h) - \R_n(h)\}$ in the same way as before, but here we have a different upper bound $M$, i.e. 
    \begin{align*}
        | \sup_{h \in \mathcal{H}} \{\R(h) - \R_n(h)\} - \sup_{h \in \mathcal{H}} \{\R(h) - \R_n'(h)\} |  &\le \sup_{h \in \mathcal{H}} | \R_n'(h) - \R_n(h) |\\
        &= \sup_{h \in \mathcal{H}} \frac{1}{n} | w_h(X_i', Y_i') - w_h(X_i, Y_i) | \\
        &\le \frac{2M}{n},
    \end{align*}
    so we get that with probability at least $1-\delta/4$
    \begin{equation*}
        \sup_{h \in \mathcal{H}} \{\R(h) - \R_n(h)\} \le \E\bigg[\sup_{h \in \mathcal{H}} \{\R(h) - \R_n(h)\}\bigg] + \sqrt{\frac{4M^2\ln{\frac{4}{\delta}}}{2n}}.
    \end{equation*}
    \item Step 2. This part is essentially identical to the respective part in \cref{theorem: infinite_hypothesis}, which combined with Step 1, gives us that with probability $1-\delta/4$
    \begin{equation*}
        \sup_{h \in \mathcal{H}} \{\R(h) - \R_n(h)\} \le 2 \mathfrak{R}(\Omega_{\mathcal{H}}) + \sqrt{\frac{4M^2\ln{\frac{4}{\delta}}}{2n}}.
    \end{equation*}
    \item Step 3. Now we need to relate the Rademacher complexity of the class:
\[
\Omega_{\mathcal{H}} = \left\{ (x,y) \mapsto \frac{d\tilde{\mathbb{P}}_h}{d\mathbb{P}}(x,y) \cdot 1_{h(x) \neq y} \mid h \in \mathcal{H} \right\}.
\]
to the Rademacher complexity of $\mathcal{H}$. We do so by first considering the empirical Rademacher complexity with respect to an i.i.d. sample $S = \{(x_1, y_1), \dots, (x_n, y_n)\}$ from $\mathbb{P}$. For each index $i$, we define the function $\phi_i: \{-1, 1\} \to \mathbb{R}$ as:
\begin{align*}
    \phi_i(u) &:= \frac{d\tilde{\mathbb{P}}_u}{d\mathbb{P}}(x_i,y_i) \cdot 1_{u \neq y_i}\\
    &= \psi(x_i, y_i, u) \cdot 1_{u \ne y_i}.
\end{align*}
The empirical Rademacher complexity is given by:
\[
\hat{\mathfrak{R}}_S(\Omega_{\mathcal{H}}) = \mathbb{E}_{\sigma} \left[ \sup_{h \in \mathcal{H}} \frac{1}{n} \sum_{i=1}^n \sigma_i \phi_i(h(x_i)) \right].
\]
We seek to bound $\hat{\mathfrak{R}}_S(\Omega_{\mathcal{H}})$ by using Talagrand's Lemma. To this end, we must determine the Lipschitz constant of $\phi_i$ with respect to its argument $u \in \{-1, 1\}$. The Lipschitz condition requires $|\phi_i(1) - \phi_i(-1)| \leq L |1 - (-1)| = 2L$.

We analyze the difference $|\phi_i(1) - \phi_i(-1)|$:
\begin{itemize}
    \item \textbf{Case 1 ($y_i = 1$):} Here $1_{1 \neq 1} = 0$ and $1_{-1 \neq 1} = 1$.
    \[
    |\phi_i(1) - \phi_i(-1)| = |0 -  \psi(x_i, 1, -1)| =  \psi(x_i, 1, -1) \leq M.
    \]
    \item \textbf{Case 2 ($y_i = -1$):} Here $1_{1 \neq -1} = 1$ and $1_{-1 \neq -1} = 0$.
    \[
    |\phi_i(1) - \phi_i(-1)| = | \psi(x_i, -1, 1) - 0| = \psi(x_i, -1, 1) \leq M.
    \]
\end{itemize}
In both cases, the difference is bounded by $M$. Therefore, $2L = M \implies L = M/2$.

Therefore, we can apply Talagrand's Contraction Lemma. Specifically, Lemma 5.7 in \cite{mohri2018foundations} states that if $\phi_i$ are $L$-Lipschitz functions, then $$\mathbb{E}_{\sigma}\left(\sup_{h\in\mathcal{H}}\frac{1}{n}\sum_{i=1}^{n}\sigma_i \phi_i(h(x_i))\right) \leq L\mathbb{E}_{\sigma}\left(\sup_{h\in\mathcal{H}}\frac{1}{n}\sum_{i=1}^{n}\sigma_i h(x_i)\right).$$

Applying this with $L = M/2$ gives:
\[
\hat{\mathfrak{R}}_S(\Omega_{\mathcal{H}}) \leq \frac{M}{2} \hat{\mathfrak{R}}_S(\mathcal{H}).
\]
Finally, we take expectation with respect to the sample, to obtain a link between the distributional Rademacher complexities. Combining this with the inequality in Step 2, we get that with with probability at least $1-\delta/4$ 
    \begin{equation*}
        \sup_{h \in \mathcal{H}} \{\R(h) - \R_n(h)\} \le M \mathfrak{R}(\mathcal{H}) + \sqrt{\frac{4 M^2\ln{\frac{4}{\delta}}}{2n}}.
    \end{equation*}

    \item Step 4. This part remains identical to that in \cref{theorem: infinite_hypothesis}, giving us that with probability at least $1-\delta/2$ 
    \begin{equation*}
        \sup_{h \in \mathcal{H}} \{\R(h) - \R_n(h)\} \le M \mathfrak{R}_n(\mathcal{H}) + (2 + \sqrt{2})M\sqrt{\frac{\ln{\frac{4}{\delta}}}{2n}},
    \end{equation*}
    \item Step 5. Lastly, bounding the $\sup_{h \in \mathcal{H}} \{ \R_n(h) - \R(h) \}$ with the same upper bound with probability $1-\delta/2$ follows in a similar way. The final result of the theorem then follows by applying the union bound. 
\end{itemize}
The rest of the theorem is essentially the same as in \cite{mohri2018foundations} and we give it here for completeness only. We have that for every $h \in \mathcal{H}$, with probability at least $1-\delta$
\begin{align*}
    \R(h_S^\ast) &\le \R_n(h_S^\ast) + M \mathfrak{R}_n(\mathcal{H}) + (2 + \sqrt{2})M\sqrt{\frac{\ln{\frac{4}{\delta}}}{2n}}\\
    &\le \R_n(h) + M \mathfrak{R}_n(\mathcal{H}) + (2 + \sqrt{2})M\sqrt{\frac{\ln{\frac{4}{\delta}}}{2n}} \\
    &\le \R(h) + 2M \mathfrak{R}_n(\mathcal{H}) + (4 + 2\sqrt{2})M\sqrt{\frac{\ln{\frac{4}{\delta}}}{2n}}.
\end{align*}
\end{proof}
  
\section{Miscellaneous}
\label{appendix: D}

\paragraph{Compute Resources}
The experiments were not computationally demanding and can be reproduced on a standard laptop or desktop machine with modest memory and processing capabilities. All computations were performed on a CPU, with no need for specialized hardware or cloud-based infrastructure.

\paragraph{Licenses} We used the Credit scoring (\cite{GiveMeSomeCredit}), and the U.S. Folktables (\cite{ding2021retiring}) datasets. These are established ML benchmarks, with the Credit data published on Kaggle subject to competition rules and the Folktables dataset released under the MIT license.